\documentclass{article}

\usepackage[accepted]{mystyle}

\usepackage{amsfonts}
\usepackage{amsthm,amsmath,natbib}
\usepackage{amssymb}
\usepackage{bbm}
\newtheorem{lemma}{Lemma}
\newtheorem{theorem}{Theorem}
\newtheorem{proposition}{Proposition}
\usepackage{dsfont}

\usepackage{hyperref}
\usepackage{enumitem}

\usepackage{microtype}
\usepackage{graphicx}
\usepackage{subfigure}
\usepackage{booktabs}
\DeclareMathOperator*{\argmin}{argmin}
\newcommand{\norm}[1]{\lVert#1\rVert}

\usepackage{color}
\usepackage[dvipsnames]{xcolor}

\usepackage{xspace}
\usepackage{booktabs}
\usepackage{multirow}
\newcommand{\ra}[1]{\renewcommand{\arraystretch}{#1}}
\usepackage{graphicx}
\usepackage[labelfont={small,bf}, textfont={small}]{caption}
\usepackage{lscape}
\usepackage{balance}
\usepackage[scaled=.7]{beramono}

\makeatother

\renewcommand{\ldots} {...}

\renewcommand*{\top}{{\mkern-1.5mu\mathsf{T}}}

\usepackage{tikz}
\newcommand{\mysetminusD}{\hbox{\tikz{\draw[line width=0.6pt,line cap=round] (3pt,0) -- (0,6pt);}}}
\newcommand{\mysetminusT}{\mysetminusD}
\newcommand{\mysetminusS}{\hbox{\tikz{\draw[line width=0.45pt,line cap=round] (2pt,0) -- (0,4pt);}}}
\newcommand{\mysetminusSS}{\hbox{\tikz{\draw[line width=0.4pt,line cap=round] (1.5pt,0) -- (0,3pt);}}}
\renewcommand{\smallsetminus}{\mathbin{\mathchoice{\mysetminusD}{\mysetminusT}{\mysetminusS}{\mysetminusSS}}}

\newcommand{\boldMath}[1]{%
    \pdfliteral direct {2 Tr 0.3 w} 
     #1%
    \pdfliteral direct {0 Tr 0 w}%
}

\newcommand{\proposedMethod}{TVI-FL\xspace}

\begin{document}

\twocolumn[

\icmltitle{Learning the piece-wise constant graph structure of a varying Ising model}

\begin{icmlauthorlist}
\icmlauthor{\ \ Batiste Le Bars}{a}
\icmlauthor{\ \ Pierre Humbert}{a}
\icmlauthor{\ \ Argyris Kalogeratos}{a}
\icmlauthor{\ \ Nicolas Vayatis}{a}%
\end{icmlauthorlist}

\icmlaffiliation{a}{Université Paris-Saclay, ENS Paris-Saclay, CNRS, Centre Borelli, F-91190 Gif-sur-Yvette, France}

\icmlcorrespondingauthor{Batiste Le Bars}{batiste.lebars@cmla.ens-cachan.fr}

\icmlkeywords{Machine Learning, Ising model, Change-point detection}

\vskip 0.3in
]

\printAffiliationsAndNotice{}

\begin{abstract}
This work focuses on the estimation of multiple change-points in a time-varying Ising model that evolves piece-wise constantly. The aim is to identify both the moments at which significant changes occur in the Ising model, as well as the underlying graph structures. For this purpose, we propose to estimate the neighborhood of each node by maximizing a penalized version of its conditional log-likelihood. The objective of the penalization is twofold: it imposes sparsity in the learned graphs and, thanks to a fused-type penalty, it also enforces them to evolve piece-wise constantly. Using few assumptions, we provide two change-points consistency theorems. Those are the first in the context of unknown number of change-points detection in time-varying Ising model. Finally, experimental results on several synthetic datasets and a real-world dataset demonstrate the performance of our method.

\end{abstract}

\section{Introduction}\label{sec:intro}

Graphs are fundamental tools to model and study static or varying relationships between variables of potentially high-dimensional vector data. They have many applications in physics, computer vision and statistics \cite{choi2010exploiting,marbach2012wisdom}. In the static scenario, learning relationships between variables is referred to as \emph{graph inference} and emerges in many fields such as in graph signal processing \citep{dong2016learning,bigbosses}, in probabilistic modeling, or in physics and biology \citep{rodriguez2011uncovering, du2012learning}. In this work, we consider a probabilistic framework where the observed data are drawn from an \emph{Ising model}, a discrete \emph{Markov Random Field} (MRF) with $\{-1,1\}$-outputs. %
MRF are undirected probabilistic graphical models \citep{koller2009probabilistic} where a set of random variables is represented as different nodes of a graph. An edge between two nodes in this graph indicates the conditional dependency between the two corresponding random variables, given the other variables.

Learning the structure of an MRF using a set of observations has been widely investigated \citep{banerjee2008model,meinshausen2006high}. In particular for Gaussian graphical models \citep{ yuan2007model,ren2015asymptotic} with the well-known graphical lasso \citep{friedman2008sparse}. The Ising model inference task has also been addressed in the past \citep{hofling2009estimation,ravikumar2010high,xue2012nonconcave,vuffray2016interaction,goel2019learning}. However, previous methods do not consider the case where the underlying structure is evolving through time.

Over the past years, there has been a burst of interest in learning the structure of time-varying MRF \cite{hallac2017network,yang2019estimating}. This task combined with the \emph{change-point detection}, which is the detection of the moments in time at which significant changes in the graph structure occur, is of particular interest. Those have been widely investigated for piece-wise constant Gaussian graphical models \citep{gibberd2017regularized, wang2018fast,londschien2019change}, in all types of the change-point detection objectives: single change-point \citep{bybee2018change}, multiple change-points \citep{gibberd2017multiple}, offline detection \citep{kolar2012estimating,gibberd2017regularized}, online detection \citep{keshavarz2018sequential}, etc. 

The advancements related to the time-varying Ising model are though limited. Especially, the combination of multiple change-points detection and structure inference has not been studied properly in the past. In \cite{ahmed2009recovering, kolar2010estimating}, the authors learn the parameters of a time-varying Ising model without looking for change-points since the network is allowed to change at each timestamp. In \cite{fazayeli2016generalized}, the authors assume that the change-point location is known and only focus on the inference of the structural changes between Ising models. More recently, the problem of detecting a single change-point has been studied in \cite{roy2017change}.

\paragraph{Contribution.} This work focuses on the estimation of multiple change-points in a time-varying Ising model that evolves piece-wise constantly. The aim is to identify both the moments at which significant changes occur in the Ising model, as well as the underlying graph structure of the model among consecutive change-points. Our work extends the work in \cite{kolar2012estimating,gibberd2017regularized} on Gaussian graphical models, to the case of an Ising model. We also derive two change-points consistency theorems that, to our knowledge, we are the first to demonstrate.  
More specifically, our method follows a ``node-wise regression" approach \cite{ravikumar2010high} and estimates the neighborhood of each node by maximizing a penalized version of its conditional log-likelihood. The penalization allows us to efficiently recover sparse graphs and, thanks to the use of a group-fused penalty \cite{harchaoui2010multiple,bleakley2011group,kolar2012estimating}, as well to recover the change-points. The proposed method is referred as \proposedMethod, which stands for Time-Varying Ising model identified with Fused and Lasso penalties.
\vspace{-1mm}
\paragraph{Organization.} The paper is organized as follows. First, we briefly recall important properties of the static Ising model and describe its piece-wise constant version. Second, we present our %
methodology for the inference of a piece-wise constant graph structure over time and the moments in time at which significant changes occur. Next, we present our main theoretical results that consist in two change-point consistency theorems. Finally, we demonstrate empirically, on multiple synthetic datasets and a real-world problem, that our method is the best suited to recover both the structure and the change-points.

\section{The time-varying Ising model}\label{sec:tempo_ising}

The \emph{static Ising model} is a discrete MRF with $\{-1,1\}$-outputs. This model is defined by a graph $G = (V,E)$ where an edge between two nodes indicates that the two corresponding random variables are dependent given the other ones. We associate this graph to a symmetric weight matrix $\Omega \in \mathbb{R}^{p\times p}$ whose non-zero elements correspond to the set of edges $E$. Formally, we have $\omega_{ab}\neq 0 $ iff $(a,b)\in E$ where $\omega_{ab}$ stands for the $(a,b)$-th element of $\Omega$. An Ising model is thus entirely described by its associated weight matrix $\Omega$. Let $X\sim \mathcal{I}(\Omega)$ be a random vector following an Ising model with weight matrix $\Omega$. Let $x\in \{-1,1\}^p$ be a realization and $x_a$, $x_b$ respectively its $a$-th and $b$-th coordinates. Then, its probability function is given by:
\begin{equation}
\label{eq:model}
\mathbb{P}_{\Omega}(X = x) = \frac{1}{Z(\Omega)}\exp\left\{\sum_{a<b}x_ax_b\omega_{ab}\right\},
\end{equation}
where $Z(\Omega)=\sum_{x\in\{-1,1\}^p}\exp\left\{\sum_{a<b}x_ax_b\omega_{ab}\right\}$ is the normalizing constant. For clarity in the following we denote $\mathbb{P}_{\Omega}(X = x) = \mathbb{P}_{\Omega}(x)$. %

A \emph{time-varying Ising model} is defined by a set of $n$ graphs $G^{(i)}=(V,E^{(i)})$, $i \in \{1,\ldots,n\}$ over a fixed set of nodes $V$ through a time-varying set of edges $\{E^{(i)}\}_{i=1}^n$. Similarly to the static case, each $G^{(i)}$ is associated to a symmetric weight matrix $\Omega^{(i)}\in \mathbb{R}^{p\times p}$ and a distribution $\mathbb{P}_{\Omega^{(i)}}$ given by Eq.\,(\ref{eq:model}). %
A random variable associated to this model is a set of $n$ independent random vectors $X^{(i)} \sim \mathcal{I}(\Omega^{(i)})$. A single realization is therefore a set of $n$ vectors, each denoted by $x^{(i)} \in \{-1,1\}^p$.

In the sequel, we assume in addition that the model is \emph{piece-wise constant}, i.e.~there exist a collection of $D$ timestamps $ \mathcal{D}\triangleq\{T_1,\ldots,T_D\} \subset \{2,\ldots,n\}$, sorted in ascending order, and a set of symmetric matrices $\{\Theta^{(j)}\}_{j=1}^{D+1}$ such that $\forall i \in \{1,\ldots,n\}$:
\begin{equation}
\label{eq:piecewise_cst}
\Omega^{(i)} = \sum_{j=0}^{D}\Theta^{(j+1)} \mathds{1}\{T_j\leq i < T_{j+1}\} ,
\end{equation}
where $T_0 = 1$, $T_{D+1} = n+1$. $\mathcal{D}$ thus corresponds to the set of change-points. According to Eq.\,(\ref{eq:piecewise_cst}), for a fixed $j \in \{0,\ldots,D\}$, the set $\left\{x^{(i)}: T_j \leq i< T_{j+1}\right\}$ contains i.i.d. vectors drown from $\mathbb{P}_{\Theta^{(j+1)}}$. %

\section{Learning Methodology}
\label{sec:learning}
Assuming the observation of a single realization $\{x^{(i)}\}_{i=1}^n$ of the described time-varying model at each timestamp, our objective is twofold. We want to recover the set of change-points $\mathcal{D}$, as well as the graph structure underlying the observed data vectors, i.e.~which edges are activated at each timestamp. %
In practice, we may observe multiple data vectors at each timestamp. %
However, since this does not change our analysis, we leave the related discussion for the experimental section. Next, we now describe our methodology to perform the aforementioned tasks.

\vspace{-1mm}
\paragraph{Neighborhood selection strategy.}
Due to the intractability of the normalizing constant $Z(\cdot)$, classical maximum likelihood approaches are difficult to apply in practice. Hence, an intuitive approach is to extend the neighborhood selection strategy introduced for the static setting in \cite{ravikumar2010high} to our time-varying setting. Instead of maximizing the global likelihood of Eq.\,(\ref{eq:model}), this approach maximizes, for each node $a\in V$, the conditional likelihood of the node knowing the other nodes in $V\smallsetminus a$. The conditional probability of observing a node's value, knowing the others, when $X\sim\mathcal{I}(\Omega)$, is:
\begin{equation}
\label{eq:conditionnal}
\mathbb{P}_{\omega_{a}}(x_a | x_{\smallsetminus a}) = \frac{\exp\left\{2x_a\sum_{b\in V\smallsetminus a}x_b\omega_{ab}\right\}}{\exp\left\{2x_a\sum_{b\in V\smallsetminus a}x_b\omega_{ab}\right\} + 1},
\end{equation}
where $\omega_{a}$ denotes the $a$-th column (or row) of $\Omega$ that is used to parametrize the probability function of Eq.\,(\ref{eq:conditionnal}). Here, $x_{\smallsetminus a}$ denotes the vector $x$ without the coordinate $a$.

 For each node, we thus propose to maximize a penalized version of the conditional likelihood of Eq.\,\ref{eq:conditionnal}. The detailed procedure is explained below.%

\subsection{Optimization program}
\label{sec:optim}

The neighborhood selection strategy works as follows. For each node $a=1,\ldots, p$, we solve the regularized optimization program:
\begin{equation}
\widehat{\beta}_a = \argmin_{\beta \in  \mathbb{R}^{p-1\times n}} \mathcal{L}_a(\beta) + pen_{\lambda_1,\lambda_2}(\beta).
\label{eq:optim}
\end{equation}
In this equation, $\mathcal{L}_a(\beta)$ stands for the node-wise negative conditional log-likelihood of node $a$, knowing $x^{(i)}_{\smallsetminus a}$: %
\begin{align} 
 \mathcal{L}_a(\beta) & \triangleq -\sum_{i=1}^n \log\left(\mathbb{P}_{\beta^{(i)}}(x^{(i)}_a | x^{(i)}_{\smallsetminus a})\right) \label{eq:loglik} \\
&= \sum_{i=1}^n\log\left\{\exp\left( \beta^{(i)\top}x^{(i)}_{\smallsetminus a}\right) + \exp\left( -\beta^{(i)\top}x^{(i)}_{\smallsetminus a}\right)\right\} \nonumber \\
& \quad \quad \quad - \sum_{i=1}^n x_a^{(i)}\beta^{(i)\top}x^{(i)}_{\smallsetminus a},
\end{align}
where $\beta^{(i)}$ is the $i$-th column of $\beta$. The last line is obtained by plugging Eq.\,(\ref{eq:conditionnal}) in Eq.\,(\ref{eq:loglik}) with $\beta^{(i)}$ instead of $\omega_{a}$.

With such objective function, we learn at each timestamp $i$ the neighborhood $\beta^{(i)}$ of node $a$ via a penalized logistic regression method. %

\vspace{-1mm}
\paragraph{Penalty term.}
Provided two hyperparameters, $\lambda_1,\lambda_2\!>\!0$, we propose the following penalty term for Eq.\,\ref{eq:optim}:
\begin{equation*}
pen_{\lambda_1,\lambda_2}(\beta) = \lambda_1\sum_{i = 2}^{n} \norm{\beta^{(i)}-\beta^{(i-1)}}_2 + \lambda_2\sum_{i = 1}^{n} \norm{\beta^{(i)}}_1.
\end{equation*}
The overall two-term penalty is necessary for recovering efficiently the piece-wise constant graph structure. The second term is quite standard: it allows the estimated parameter vectors to be sparse and thus imposes structure in the learned graphs. %
On the other hand, without the first term, we would fit for each timestamp $i\in \{1,\ldots,n\}$ a parameter vector $\beta^{(i)}$ that perfectly matches the observation $x^{(i)}$ (in terms of likelihood). In such a situation, we would obtain as many different parameter vectors $\beta$ as there is different samples, making the piece-wise constant assumption of Eq.\,\ref{eq:piecewise_cst} impossible to recover. This is why we propose a group-fused penalty, consisting in the $\ell_2$-norm of the difference between two consecutive parameter vectors. The sum of the $\ell_2$-norms acts as a group-lasso penalty on temporal difference between consecutive parameter vectors, which encourages the two vectors to be equal. This allows us to learn efficiently an evolving piece-wise constant structure and also to detect the change-points. %

In conclusion, the hyperparameter $\lambda_1$ controls the number of estimated change-points -- the larger $\lambda_1$ is, the fewer the estimated number of change-points will be. Similarly, when $\lambda_2$ increases, the sparsity of each parameter vector increases as well. A priori, choosing the hyperparameters is not straightforward. However, since our objective function corresponds to a penalized logistic regression problem, we can use existing model selection criteria. %
We discuss further about this aspect in the experimental section.

\subsection{Change-point detection and structure estimation}

Assume that the optimization program (\ref{eq:optim}) is solved. The set of estimated change-points $\widehat{\mathcal{D}}$ is: 
\begin{equation*}
    \widehat{\mathcal{D}}=\left\{\widehat{T}_j \in \{2,\ldots,n\} : \norm{\widehat{\beta}_a^{(\widehat{T}_j)}-\widehat{\beta}_a^{(\widehat{T}_j-1)}}_2 \neq 0 \right\} .
\end{equation*}
Namely, this corresponds to the set of timestamps at which the estimated parameter vectors have changed. %
For each submodel $j = 1,\ldots, |\widehat{\mathcal{D}}|+1$, the $a$-th column of $\Theta^{(j)}$ is estimated by $\widehat{\theta}_{a}^{(j)}\triangleq \widehat{\beta}_a^{( \widehat{T}_{j-1})}\!=\ldots=\widehat{\beta}_a^{( \widehat{T}_j-1)}$. The non-zero elements of $\widehat{\theta}_{a}^{(j)}$ indicate the \emph{neighborhood} of $a$.

One should notice that this estimation leads to a non-symmetric weight matrix. To overcome this problem, is was proposed in \cite{ravikumar2010high,kolar2012estimating} to either use the $\min$ or $\max$ operator. In the present work, to estimate the structure of the $j$-th graph, we take:
\begin{equation*}
    \widehat{E}^{(j)} = \{(a,b) : \max(|\widehat{\theta}_{ab}^{(j)}|,|\widehat{\theta}_{ba}^{(j)}|)>0\},
\end{equation*}
where $\widehat{\theta}_{ab}^{(j)}$ is the $b$-th element of $\widehat{\theta}_{a}^{(j)}$, and conversely for $\widehat{\theta}_{ba}^{(j)}$. In this case, there is an edge between two nodes if at least one of them contains the other node in its neighborhood.

\section{Theoretical analysis}

In this section, we present two change-point consistency theorems for \proposedMethod. The theorems state that, as the number of samples $n$ tends to infinity, the change-points are estimated more and more precisely. %

\subsection{Technical assumptions}
We denote by $\widehat{D} = |\widehat{\mathcal{D}}|$ (the set's cardinality) the total number of detected change-points, respectively for the real changes $D = |\mathcal{D}|$, and by $\left[D\right]$ the set of indices $\{1,\ldots,D\}$. Let us now define two important quantities. The first is the minimal time difference between two change-points:%
\begin{equation*}
\Delta_{\min}\triangleq \min_{j\in \left[D\right]}|T_{j}-T_{j-1}| .
\end{equation*}
The second quantity is the minimal variation in the model parameters between two change-points, which is given by:
\begin{equation*}
\xi_{\min}\triangleq \min_{a\in V,j\in \left[D\right]}\lVert \theta^{(j+1)}_a - \theta^{(j)}_a \rVert_2 .
\end{equation*}
We now introduce three standard assumptions on the Ising model inference and change-points detection. They are assumed to be true for each node $a \in V$.
\noindent

\textbf{(A1)} \textit{There exist two constants $\phi_{\min}>0$ and $\phi_{\max}<\infty$ such that $\forall j \in \left[ D+1\right]$,  $\phi_{\min}\leq\Lambda_{\min} \left(\mathbb{E}_{\Theta^{(j)}}[X_{\smallsetminus a}X_{\smallsetminus a}^{\top} ]\right)$ and $\phi_{\max}\geq\Lambda_{\max} \left(\mathbb{E}_{\Theta^{(j)}}[X_{\smallsetminus a}X_{\smallsetminus a}^{\top} ]\right)$. Here $\Lambda_{\min}(\cdot)$ and $\Lambda_{\max}(\cdot)$ denote, respectively, the smallest and largest eigenvalues of the input matrix}.\\
\vspace{-0.7em}\\
 This is a standard assumption for such problems: it ensures that the covariates are not too dependent, and makes the model identifiable \citep{ravikumar2010high,kolar2012estimating}. In fact, this assumption is always verified if the support of the model is sufficiently large. Indeed, if at least $p$ linearly independent vectors have a non-zero probability to be observed, then the matrix $\mathbb{E}_{\Theta^{(j)}}[X_{\smallsetminus a}X_{\smallsetminus a}^{\top}]$ will have full rank.\\
\vspace{-0.7em}\\
\noindent
\textbf{(A2)} \textit{There exists a constant $C>0$ such that $\max_{j,l \in [D+1]}\lVert \theta^{(j)}_a - \theta^{(l)}_a \rVert_2\leq C$, and a constant $M$ such that $\max_{j\in \left[D+1\right]}\lVert \theta^{(j)}_a \rVert_2\leq M$.} \\
\vspace{-0.7em}\\
\noindent
\textbf{(A3)} \textit{The sequence $\{T_j\}_{j=1}^D$ satisfies, for each $j$, $T_j = \lfloor n\tau_j \rfloor$ , where $\lfloor x \rfloor$ is the largest integer smaller than or equal to $x$  and $\{\tau_j\}_{j=1}^D$ is a fixed, unknown sequence of the change-point fractions belonging to $[0,1]$.} \\
\vspace{-0.7em}\\
This last assumption says that as $n$ grows, the new observations are sampled uniformly across all the $D+1$ sub-models. \\

\subsection{Main results}
Next, we present our theoretical results on \textit{change-point consistency}. The proofs are made for one node, $a$, but generalize to all the other nodes.
We first provide the optimality conditions necessary to demonstrate the main results.

\begin{lemma}{(Optimality Conditions)}
\label{lemma:optim_cond}
A matrix $\hat{\beta}$ is optimal for problem (\ref{eq:optim}) iff there exists a collection of subgradient vectors $\{\hat{z}^{(i)}\}_{i=2}^n$ and $\{\hat{y}^{(i)}\}_{i=1}^n$, with $\hat{z}^{(i)} \in \partial \norm{\widehat{\beta}^{(i)}-\widehat{\beta}^{(i-1)}}_2$ and $\hat{y}^{(i)} \in \partial \norm{\widehat{\beta}^{(i)}}_1$, such that $\forall k=1,\ldots,n$:
\begin{align}
\sum_{i=k}^n & x^{(i)}_{\smallsetminus a}\left\{\tanh\!\left( \widehat{\beta}^{(i)\top}x^{(i)}_{\smallsetminus a})\right) - \tanh\!\left( \omega_a^{(i)\top}x^{(i)}_{\smallsetminus a})\right)\right\} \nonumber \\
& - \sum_{i=k}^n x^{(i)}_{\smallsetminus a}\left\{x^{(i)}_a - \mathbb{E}_{\Omega^{(i)}}\left[X_a | X_{\smallsetminus a} = x^{(i)}_{\smallsetminus a} \right]\right\} \nonumber \\ 
& + \lambda_1\hat{z}^{(k)} + \lambda_2\sum_{i=k}^n\hat{y}^{(i)} = \bold{0}_{p-1},
\end{align}
\text{where} $\tanh$ is the hyperbolic tangent function, $\bold{0}_{p-1}$ is the %
zero vector of size $p-1$, $\hat{z}^{(1)}=\bold{0}_{p-1}$, and 

$\hat{z}^{(i)} = \left\{\begin{array}{l}
\frac{\widehat{\beta}^{(i)}-\widehat{\beta}^{(i-1)}}{\norm{\widehat{\beta}^{(i)}-\widehat{\beta}^{(i-1)}}_2}\quad \text{ if \ } \widehat{\beta}^{(i)}-\widehat{\beta}^{(i-1)} \neq 0, \\ 
\in \mathcal{B}_2(0,1) \quad \quad\ \ \,\text{otherwise},
\end{array} \right.$ \\ $\hat{y}^{(i)} = \left\{\begin{array}{l} \text{\emph{sign}}(\widehat{\beta}^{(i)}) \quad\quad\quad \text{if \ } x \neq 0, \\
\in \mathcal{B}_1(0,1) \quad\quad\ \ \  \text{otherwise}.\end{array}\right.$
\end{lemma}

\begin{proof}
The proof is given in the Appendix. It consists in writing the sub-differential of the objective function and say, thanks to the convexity, that $0$ belongs to it.
\end{proof}
\begin{theorem}{(Change-point consistency)} Let $\{x_i\}_{i=1}^n$ be a sequence of observations drawn from the model presented in Sec.\,\ref{sec:tempo_ising}. Suppose (A1-A3) hold, and assume that $\lambda_1 \asymp \lambda_2 = \mathcal{O}(\sqrt{\log(n)/n})$. Let $\{\delta_n\}_{n\geq 1}$ be a non-increasing sequence that converges to $0$, and such that $\forall n>0$, $\Delta_{\min}\geq n\delta_n$, with $n\delta_n \rightarrow +\infty$. Assume further that $(i)$ $\frac{\lambda_1}{n\delta_n\xi_{\min}}\rightarrow 0$, $(ii)$ $\frac{\sqrt{p-1}\lambda_2}{\xi_{\min}}\rightarrow 0$, and $(iii)$ $\frac{\sqrt{p\log(n)}}{\xi_{\min}\sqrt{n\delta_n}}\rightarrow 0$. Then, if the correct number of change-points are estimated, we have $\widehat{D}=D$ and:
\begin{equation}
    \mathbb{P}(\max_{j=1,\ldots,D}|\hat{T}_j-T_j|\leq n\delta_n)\underset{n\rightarrow \infty}{\longrightarrow}1 .
    \label{eq:consitency_thme}
\end{equation}
\label{thme}
\end{theorem}
\begin{proof}
We extend the proof given in \cite{harchaoui2010multiple,kolar2012estimating} to the particular case of the Ising model. While some important steps are essentially similar, the main difference between our proof and the previous one are Lemmas $2$, $3$, $4$, and $5$ included in the Appendix, which provide concentration bounds adapted to the Ising model setting. Moreover, we had to employ several new tricks that were not used in the previous papers. We give here a sketch of the proof.

Thanks to the union bound, the probability of the complementary in Eq.\,(\ref{eq:consitency_thme}) can be upper bounded by:
\begin{equation*}
    \mathbb{P}(\max_{j=1,\ldots,D}|\hat{T}_j-T_j|>n\delta_n) \leq \sum_{j=1}^D \mathbb{P}(|\hat{T}_j-T_j|>n\delta_n).
\end{equation*}
To prove Eq.\,(\ref{eq:consitency_thme}), it is now sufficient to show $\forall j=1,\ldots,D$ that $\mathbb{P}(|\hat{T}_j-T_j|>n\delta_n) \rightarrow 0$. Let us define the event $C_n=\{|\hat{T}_j - T_j|<\frac{\Delta_{\min}}{2}\}$ and its complementary $C_n^c$. The rest of the proof is divided in two parts: bounding the good scenario, i.e.~show that $\mathbb{P}(\{|\hat{T}_j-T_j|>n\delta_n\}\cap C_n) \rightarrow 0$, and doing the same for the bad scenario, i.e $\mathbb{P}(\{|\hat{T}_j-T_j|>n\delta_n\}\cap C_n^c) \rightarrow 0$. 

To bound the good scenario, the proof applies Lemma\,\ref{lemma:optim_cond} to bound the considered probability by three others probabilities. These latter are then asymptotically bounded by $0$, thanks to a combination of Assumptions\,(A1-A3), assumptions of the theorem and concentration inequalities related to the considered time-varying Ising model (given by the lemmas of the Appendix).

To bound the bad case scenario, the three following complementary events are defined:
\begin{align*}
    & D_n^{(l)} \triangleq \left\{\exists j \in [D], \widehat{T}_j \leq T_{j-1} \right\} \cap C_n^c, \\
    & D_n^{(m)} \triangleq \left\{\forall j \in [D], T_{j-1} < \widehat{T}_j < T_{j+1} \right\} \cap C_n^c, \\
    & D_n^{(r)} \triangleq \left\{\exists j \in [D], \widehat{T}_j \geq T_{j+1} \right\} \cap C_n^c.
\end{align*}
Thus, it suffices to prove that $\mathbb{P}(\{|\hat{T}_j-T_j|>n\delta_n\}\cap D_n^{(l)})$, $\mathbb{P}(\{|\hat{T}_j-T_j|>n\delta_n\}\cap D_n^{(m)})$, and $\mathbb{P}(\{|\hat{T}_j-T_j|>n\delta_n\}\cap D_n^{(r)}) \rightarrow 0$ as $n \rightarrow \infty$. To prove this, similar arguments to those used for the good case are employed.
\end{proof}

Note that with $\delta_n=\log(n)^\gamma/n$, for any $\gamma>1$ and $\xi_{\min}=\Omega(\sqrt{\log(n)/\log(n)^\gamma})$, the conditions of the theorem are met. With this parameterization, we obtain a convergence rate of order $\mathcal{O}(\log(n)^\gamma/n)$ for the estimation of the change-points. More precisely, for any $\delta > 0$ and sufficiently large $n$, we have with probability at least $1-\delta$ that  
\begin{equation*}
    \frac{1}{n}\max_{j=1,\ldots,D}|\hat{T}_j-T_j|\leq \frac{1}{n}\log(n)^\gamma.
\end{equation*}

In conclusion, we obtain the same rate of convergence to that of the single change-point detection method given in \cite{roy2017change}. It is almost optimal up to a logarithmic factor. 
The main drawback of the previous theorem is that it assumes that the number of change-points have been correctly estimated. In practice this is complicated to verify, while proving that the right number of change-points are consistently estimated is also difficult to get for this type of methods \cite{harchaoui2010multiple}. Nevertheless, in practice we may have an idea about an upper bound on the true number of change-points.

The next proposition provides a consistency result when the number of change-points is overestimated. Let us first introduce the metric $d(A\| B)$ defined as:
\begin{equation}
    d(A\| B) = \sup_{b\in B} \inf_{a \in A} |b-a|.
\end{equation}
\begin{proposition}
 Let $\{x_i\}_{i=1}^n$ be a sequence of observations drawn from the model presented in Sec.\,\ref{sec:tempo_ising}. Assume the conditions of Theorem\,\ref{thme} are respected. Then, if for a fix $D_{\max}<\infty$, we have $D \leq \widehat{D}\leq D_{\max}$  then:
 \begin{equation*}
     \mathbb{P}(d(\mathcal{\widehat{D}}\| \mathcal{D})\leq n\delta_n)\underset{n\rightarrow \infty}{\longrightarrow}1.
 \end{equation*}
\end{proposition}
\begin{proof}
A detailed proof is provided in Appendix. The proof applies multiple times the different tricks used to prove Theorem\,\ref{thme} and the Lemmas  also given in the Appendix. 
\end{proof}
Proposition\,1 is of fundamental importance as it tells us that, even though the number of change-points has been overestimated, asymptotically, all the true change-points belong to the set of estimated change-points. %

\section{Experimental study}
\label{seq:xp}

This section provides numerical arguments showing the empirical performance of \proposedMethod. All the experiments were implemented using Python and conducted on a personal laptop. The code of \proposedMethod is provided in the supplementary material so as a Jupyter Notebook reproducing results and figures of the real-world example. 

\subsection{Optimization procedure}\label{sec:exps-optimization}

Despite being non-differentiable, the convexity of the objective function allows the use of existing convex optimization algorithms of the literature. In this work, we use the python package {\small\texttt{CVXPY}} \citep{cvxpy} that allows us to solve our problem efficiently. Note also that the optimization for each node is independent to the other nodes, and hence the approach allows efficient parallel implementations. %

In the situation where more than one data vector is observed at each timestamp, one has simply to replace the node-wise negative log-likelihood in Eq.~\ref{eq:loglik} with:
\begin{equation}
     -\sum_{i=1}^n\sum_{l=1}^{n^{(i)}} \log\left(\mathbb{P}_{\beta^{(i)}}(x^{(il)}_a | x^{(il)}_{\smallsetminus a})\right)\!,
     \label{eq:multiple_obs}
\end{equation}
where $n^{(i)}$ stands for the number of data vectors observed at timestamp $i$, and $x^{(il)}$ for the $l$-th observed vector at time $i$.

\vspace{-1mm}
\paragraph{Tuning the hyperparameters $\lambda_1$ and $\lambda_2$.}  As stated in Sec.\,\ref{sec:optim}, it is possible to employ any model selection technique suited for logistic regression. In the experiments, we use and compare two techniques. The first is the \emph{Akaike Information Criterion} (AIC) that computes the average of the following quantity for all nodes:
\begin{equation} \label{eq:BIC}
\text{AIC}(\widehat{\beta}_a)\triangleq
2\mathcal{L}_a (\widehat{\beta}_a) +
2 \, \text{Dim}(\widehat{\beta}_a).
\end{equation}
\vspace{-5mm}
\begin{align*}
\text{Dim}(\widehat{\beta}_a) = \sum_{i=1}^n\! \left(\!\mathds{1}{\{\widehat{\beta}_{a}^{(i)} \!\neq \widehat{\beta}_{a}^{(i-1)}\}}\!\!\sum_{b \in V \smallsetminus a}\!\!\mathds{1}{\{\text{sign}(\widehat{\beta}_{ab}^{(i)}) \!\neq 0\}}\!\right)\!
\end{align*}
counts the number of parameters that are estimated. By convention, $\widehat{\beta}_{a}^{(0)} = \widehat{\beta}_{a}^{(1)}$. In this case, set of hyperparameters that minimize the AIC are finally selected.

The second technique, based on cross-validation (CV), assumes that more than one sample is observed at each moment in time $i=1,\ldots,n$. Thus, the time-series can be split in a part for the learning phase and another part for the testing phase. In our experiments, we selected the hyperparameters maximizing the AUC i.e.~the area under the ROC-curve associated to the classification score (the probability to be equal to either $1$ or $-1$). 

For both model selection techniques, AIC and CV, the hyperparameters are found using either the standard random-search or grid-search strategies.

\subsection{Experimental setup}
\label{sec:baseline}

\begin{table*}[t]\scriptsize%
    \centering
	\ra{1.0}
	\begin{tabular}{@{}lll|llr|llr@{}}
		\toprule
         \textbf{} & \textbf{Observations} & & \multicolumn{3}{c}{\textbf{AIC}} & \multicolumn{3}{c}{\textbf{AUC}} \\
		\textbf{Degree} & \textbf{per timestamps} & \textbf{Method}  & \textbf{$h$-score} $\downarrow$ & \textbf{$F_1$-score} $\uparrow$ & $\boldMath{\widehat{D}\hfill}$ & \textbf{$h$-score} $\downarrow$ & \textbf{$F_1$-score} $\uparrow$ & $\boldMath{\widehat{D}\hfill}$\\ \midrule
		%
		\multirow{2}{*}{$d=2$} & \multirow{2}{*}{$n^{(i)}=4$} & TVI-FL   & $\boldMath{0.046 \pm (0.024)}$ & $\boldMath{0.694 \pm (0.103)}$ & $7.400 \pm (3.137)$ & $0.221 \pm (0.035)$ & $\boldMath{0.876 \pm (0.030)}$ & $26.100 \pm (7.739)$ \\
		&& Tesla & $0.106 \pm (0.087)$ & $0.649 \pm (0.190)$ & $12.700 \pm (7.682)$ & $\boldMath{0.184 \pm (0.051)}$ & $0.841 \pm (0.041)$ & $25.100 \pm (4.784)$\\
		\cmidrule{2-9}
		\multirow{2}{*}{} & \multirow{2}{*}{$n^{(i)}=6$} & TVI-FL   &  $\boldMath{0.129 \pm (0.058)}$ & $\boldMath{0.816 \pm (0.073)}$ & $9.700 \pm (2.759)$ & $\boldMath{0.147 \pm (0.071)}$ & $\boldMath{0.875 \pm (0.027)}$ & $15.300 \pm (3.378)$ \\
		&& Tesla &  $0.178 \pm (0.130)$ & $0.748 \pm (0.167)$ & $12.900 \pm (5.540)$ & $0.164 \pm (0.062)$ & $0.841 \pm (0.048)$ & $19.000 \pm (2.530)$ \\
		\cmidrule{2-9}
		\multirow{2}{*}{} & \multirow{2}{*}{$n^{(i)}=8$} & TVI-FL   &  $\boldMath{0.082 \pm (0.081)}$ & $0.833 \pm (0.095)$ & $7.400 \pm (3.040)$ & $\boldMath{0.099 \pm (0.073)}$ & $\boldMath{0.891 \pm (0.024)}$ & $11.000 \pm (3.873)$ \\
		&& Tesla &  $0.124 \pm (0.071)$ & $\boldMath{0.846 \pm (0.047)}$ & $13.600 \pm (2.010)$ & $0.178 \pm (0.066)$ & $0.853 \pm (0.039)$ & $14.700 \pm (3.348)$ \\
		\midrule
		%
		\multirow{2}{*}{$d=3$} & \multirow{2}{*}{$n^{(i)}=4$} & TVI-FL   & $\boldMath{0.080 \pm (0.069)}$ & $\boldMath{0.563 \pm (0.089)}$ & $7.000 \pm (2.683)$ & $\boldMath{0.204 \pm (0.035)}$ & $\boldMath{0.734 \pm (0.024)}$ & $23.100 \pm (6.715)$ \\
		&& Tesla & $0.278 \pm (0.319)$ & $0.353 \pm (0.072)$ & $3.200 \pm (2.891)$ & $0.208 \pm (0.029)$ & $0.611 \pm (0.041)$ & $29.200 \pm (3.187)$ \\
		\cmidrule{2-9}
		\multirow{2}{*}{} & \multirow{2}{*}{$n^{(i)}=6$} & TVI-FL   & $\boldMath{0.055 \pm (0.064)}$ & $\boldMath{0.617 \pm (0.161)}$ & $6.300 \pm (3.494)$ & $\boldMath{0.130 \pm (0.051)}$ & $\boldMath{0.743 \pm (0.034)}$ & $12.800 \pm (2.821)$ \\
		&& Tesla & $0.302 \pm (0.241)$ & $0.346 \pm (0.060)$ & $2.000 \pm (1.183)$ & $0.173 \pm (0.044)$ & $0.616 \pm (0.041)$ & $22.600 \pm (2.245)$ \\
		\cmidrule{2-9}
		\multirow{2}{*}{} & \multirow{2}{*}{$n^{(i)}=8$} & TVI-FL   & $\boldMath{0.091 \pm (0.073)}$ & $\boldMath{0.714 \pm (0.130)}$ & $8.000 \pm (2.530)$ & $\boldMath{0.127 \pm (0.073)}$ & $\boldMath{0.764 \pm (0.032)}$ & $10.400 \pm (2.154)$ \\
		&& Tesla & $0.311 \pm (0.231)$ & $0.361 \pm (0.098)$ & $2.600 \pm (2.615)$ & $0.162 \pm (0.052)$ & $0.633 \pm (0.045)$ & $18.700 \pm (3.716)$ \\
		\midrule
		%
		\multirow{2}{*}{$d=4$} & \multirow{2}{*}{$n^{(i)}=4$} & TVI-FL   & $\boldMath{0.101 \pm (0.082)}$ & $\boldMath{0.453 \pm (0.111)}$ & $6.500 \pm (3.324)$ & $\boldMath{0.232 \pm (0.026)}$ & $\boldMath{0.644 \pm (0.041)}$ & $29.400 \pm (4.317)$\\
		&& Tesla & $0.444 \pm (0.273)$ & $0.347 \pm (0.044)$ & $2.875 \pm (1.900)$ & $0.234 \pm (0.017)$ & $0.518 \pm (0.046)$ & $34.625 \pm (1.654)$\\
		\cmidrule{2-9}
		\multirow{2}{*}{} & \multirow{2}{*}{$n^{(i)}=6$} & TVI-FL  & $\boldMath{0.099 \pm (0.064)}$ & $\boldMath{0.501 \pm (0.130)}$ & $5.667 \pm (2.309)$ & $\boldMath{0.183 \pm (0.044)}$ & $\boldMath{0.664 \pm (0.041)}$ & $16.778 \pm (3.258)$ \\
		&& Tesla & $0.258 \pm (0.236)$ & $0.355 \pm (0.035)$ & $2.500 \pm (1.118)$ & $0.215 \pm (0.032)$ & $0.503 \pm (0.040)$ & $26.000 \pm (4.472)$ \\
		\cmidrule{2-9}
		\multirow{2}{*}{} & \multirow{2}{*}{$n^{(i)}=8$} & TVI-FL  & $\boldMath{0.077 \pm (0.076)}$ & $\boldMath{0.528 \pm (0.158)}$ & $5.556 \pm (3.624)$ & $\boldMath{0.169 \pm (0.064)}$ & $\boldMath{0.678 \pm (0.049)}$ & $12.444 \pm (4.524)$ \\
		&& Tesla & $0.251 \pm (0.230)$ & $0.357 \pm (0.044)$ & $2.625 \pm (0.696)$ & $0.219 \pm (0.027)$ & $0.518 \pm (0.054)$ & $24.000 \pm (2.398)$ \\
		\bottomrule
	\end{tabular}
	\caption{Results for the model with the lowest AIC, and that with the highest AUC. The average $\pm$\,(std) of the metrics is reported. The best results between TVI-FL and Tesla, for each metric, are specified in bold.}
    \label{table:res_BIC}
\end{table*}

\paragraph{Baseline method.} As mentioned in Sec.\,\ref{sec:intro}, no existing work in the literature deals properly with the considered multiple change-points detection task. Several methods deal with varying Gaussian graphical models \cite{kolar2012estimating, yang2019estimating}, varying Ising models with smooth structural changes over time \cite{kolar2010estimating}, or the detection of a single change-point in the varying Ising model \cite{roy2017change}. The closest work we can compare with is the Tesla method \cite{ahmed2009recovering,kolar2010estimating}. Its major difference to our approach is the use of the $\ell_1$-norm instead of the $\ell_2$-norm as a fused-penalty. This difference is very significant, theoretically and practically. 

Indeed, \emph{using an $\ell_1$-norm fused-penalty does not encourage the recovery of a graphical model that evolves piece-wise constantly as a whole}, which makes it less adaptable to recover change-points. More specifically, such a term does not encourage two consecutive parameter vectors to be equal at every dimensions: the regularization only affects each dimension independently. Thus, despite the edge weights may evolve \emph{independently} in a piece-wise constant fashion, those changes occur at arbitrary timestamps and does not aggregate to a globally piece-wise constant behavior. An illustration of this phenomena and a comparison with the $\ell_2$-norm can be found in the Appendix. 
Nonetheless, the same way the standard linear regression can be used to recover sparse parameters, Tesla can still be used to recover change-points in practice. Hence, we choose this method as our baseline because, despite the lack of any theoretical guarantee, it can still be applicable, provided a sufficiently large sample size and appropriately tuned regularization.
\vspace{-1mm}
\paragraph{Performance metrics.} We use two suitable metrics to evaluate the quality of \proposedMethod on the learned graphs and change-points. The first one, very standard in change-point detection tasks \cite{truong2019selective} and known as the \emph{Hausdorff metric}, measures the longest temporal distance between a change-point in $\mathcal{D}$ and its prediction in $\widehat{\mathcal{D}}$:
\begin{equation*}
    h(\mathcal{D},\widehat{\mathcal{D}}) \triangleq \frac{1}{n} \max \left\{\underset{t \in \mathcal{D}}{\max}\text{ } \underset{\hat{t} \in \widehat{\mathcal{D}}}{\min} |t - \hat{t}|\text{, }\, \underset{\hat{t} \in \widehat{\mathcal{D}}}{\max}\text{ } \underset{t \in \mathcal{D}}{\min} |t - \hat{t}|\right\}\!.
\end{equation*}
The lower this metric is, the better is the estimation. %
The second one, the \emph{$F_1$-score}, measures the goodness of the learned graphs structures (high value is better) by the quantity:
\begin{align*}
F_1 &= \frac{2 \times precision \times recall}{precision + recall},
\end{align*}
which combines the two following classic measures:
\begin{align*}
precision &= \frac{1}{n}\sum_{i=1}^n\sum_{a<b}\frac{\mathds{1}{\footnotesize\{(a,b)\in\hat{E_i} \land (a,b)\in E_i\}}}{\mathds{1}{\{(a,b)\in\hat{E_i}\}}}, \\
recall &= \frac{1}{n}\sum_{i=1}^n\sum_{a<b}\frac{\mathds{1}{\{(a,b)\in\hat{E_i} \land (a,b)\in E_i\}}}{\mathds{1}{\{(a,b)\in E_i\}}}.
\end{align*}

\subsection{Application to synthetic data}

\paragraph{Simulation design.} We compare the performance of our method \proposedMethod against Tesla using several independent synthetic datasets. %
We first fix certain characteristics for all generated datasets: each of them has $n=100$ timestamps, $|D|=2$ change-points at the $51$-st and $81$-st timestamps, hence resulting in $3$ submodels being valid respectively for $50$, $30$, and $20$ timestamps. %
We consider the graph structure of each submodel to be an independent random $d$-regular graph of $p=20$ nodes, where at each time the degree of the all nodes can be $d \in \{ 2, 3, 4\}$.

To generate a piece-wise constant Ising model:%
\begin{itemize}[leftmargin=5.0mm, topsep=-1mm,itemsep=-0.8mm]
    \item 
    We first pick a degree value $d \in \{ 2, 3, 4\}$ and draw independently $3$ random $d$-regular graphs, one for each submodel. Same as in \cite{ahmed2009recovering}, the edge weights are drawn from a uniform distribution taken over $[-1,-0.5]\cup [0.5,1]$.
    \item 
    For each submodel, we draw observations using Gibbs sampling with a burn-in period of $1000$ samples. Moreover, we collect one observation every $20$ samples (lag) to avoid dependencies between them. In fact, instead of a single observation, for each timestamp $i\in\{1,\ldots,n\}$ we generate multiple observations $n^{(i)}$ in $\{4,6,8\}$, which requires to use the likelihood of Eq.\,(\ref{eq:multiple_obs}). Besides, to be able to perform CV, we also sample $5$ more observations per timestamp and use them only in the testing phase.
\end{itemize}
\vspace{0.7mm}
With the above procedure we generate $10$ different piece-wise constant Ising models for each degree $d$, which makes $30$ models to learn in total. In addition, for each model, we generate $3$ different sets of observations, one for each $n^{(i)}\in \{4,6,8\}$, that constitute the individual learning problems of our evaluation. This results in $90$ experiments in total.

For each experiment, we use a random-search strategy to find the best pair of hyperparameters $(\lambda_1, \lambda_2)$ in $[4, 15] \times [30, 40]$. This is done individually for the \proposedMethod and Tesla methods. %
The selected hyperparameters are those minimizing the AIC or maximizing the AUC (see Sec.\,\ref{sec:exps-optimization}).

\paragraph{Results.} %
The average value and standard deviation of the corresponding $h$-score and $F_1$-score over each group of $10$ experiments are reported in Tab.\,\ref{table:res_BIC}. 
The results clearly show that \proposedMethod outperforms Tesla, regardless which model selection criterion we consult. This was expected as Telsa is not designed to recover Ising models that are evolving piece-wise constantly (see Sec.\,\ref{sec:baseline}). Furthermore, while in some cases Tesla finds a number of change-points closer to the true number, the associated $h$-scores are still higher than those of \proposedMethod.
Yet, Tesla is still not irrelevant to the task and in fact there are cases in which it reaches competitive performance scores to those of \proposedMethod. %
Another finding %
is that AIC seems to favor a low number of estimated changes-points. It achieves better $h$-scores for this simulated process, while the AUC criterion seems to give priority to the recovery of the graph structure, illustrated by higher $F_1$-scores. 

We show that empirically it is possible to obtain both low $h$-score and high $F_1$-score via better hyperparameters tuning. Specifically, for each experiment and for each degree $d$, we select the model with the highest $F_1$-score when the associated $h$-score $\leq h_{\text{min}}$, with $h_{\text{min}} \in \{0, .01, .02, .03\}$. This allows, respectively at most $0$ to $3$ timestamps of offset error between an estimated and a real change-point.
In the results of Fig.\,\ref{fig:barplot} we observe that even with very low $h$-score, high $F_1$-score are reachable. Furthermore, the \proposedMethod method always provides better $F_1$-score %
than Tesla, confirming once again its superior performance. 
 
\begin{figure}[h]
\centering
\includegraphics[width=.45\linewidth]{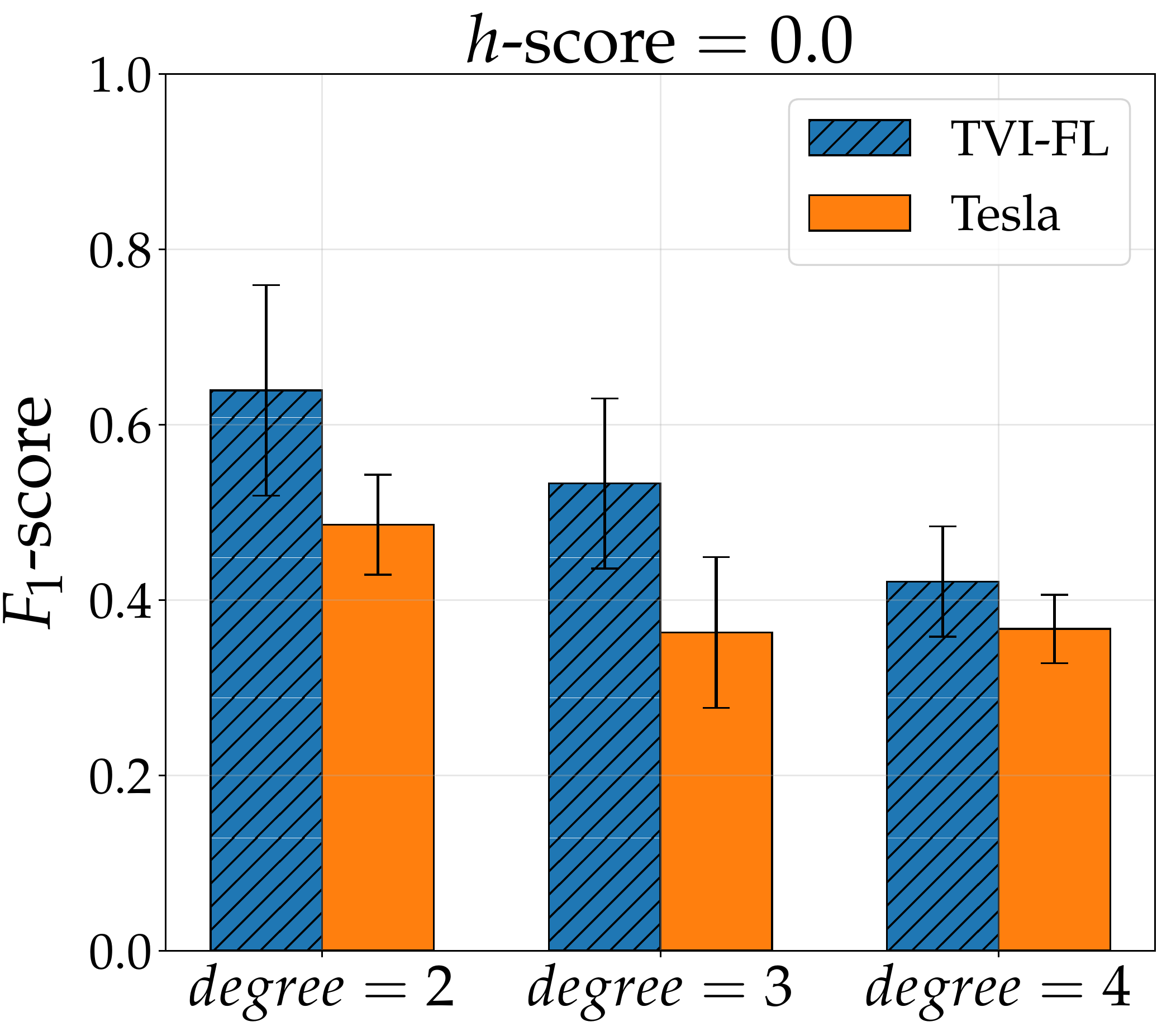}
\hspace{0.7em}
\includegraphics[width=.45\linewidth]{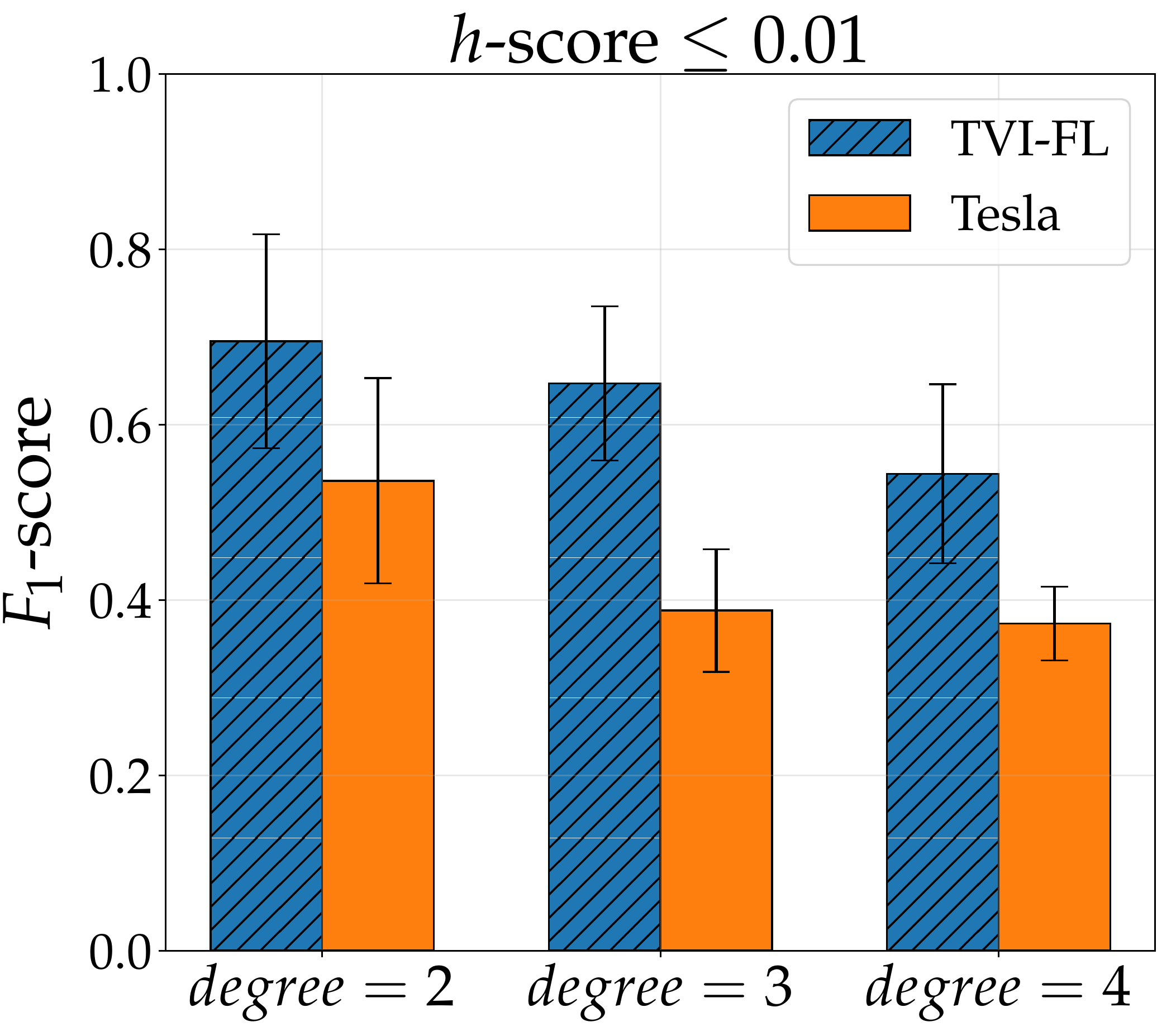}\\
\vspace{0.6em}
\includegraphics[width=.45\linewidth]{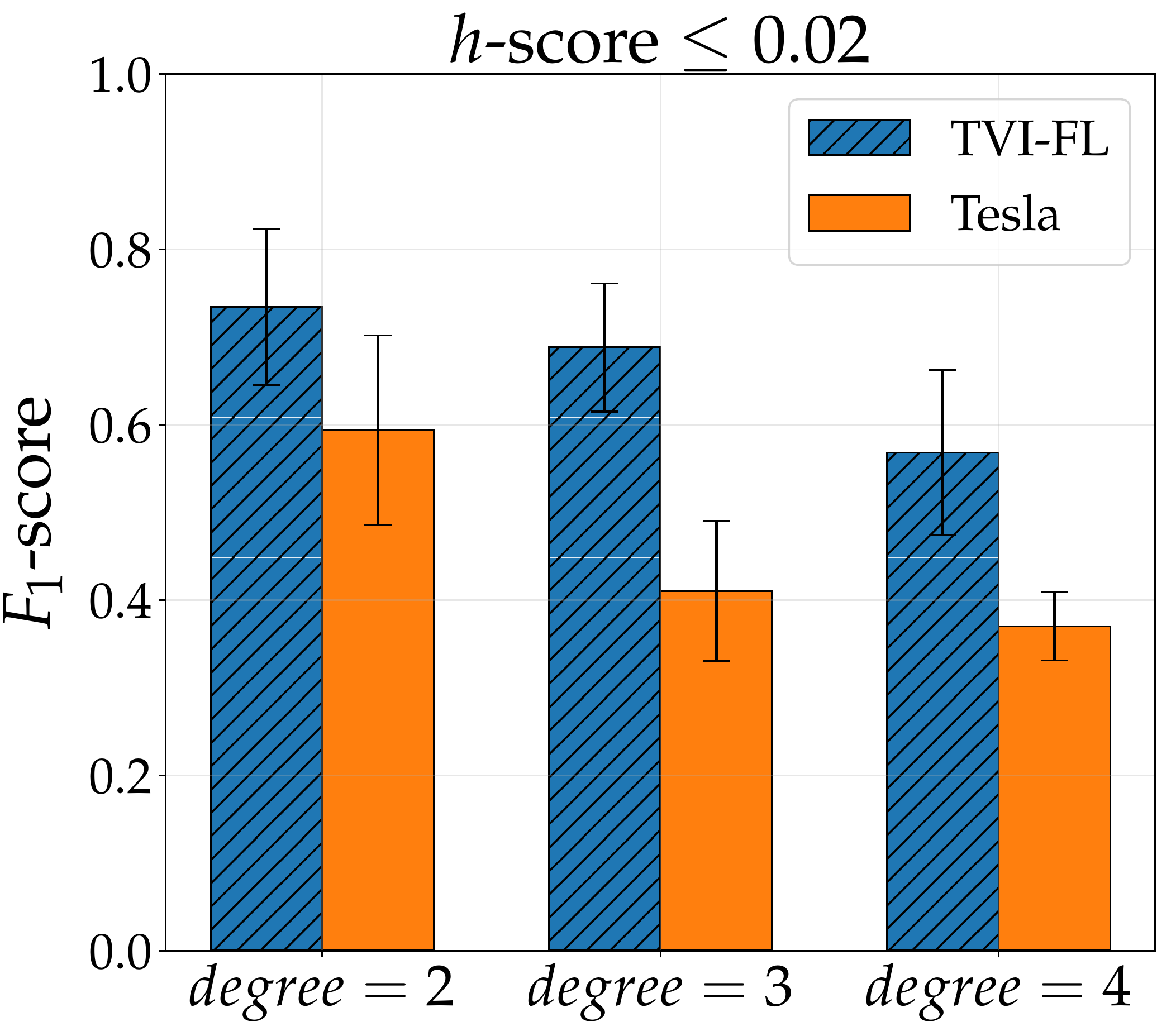}
\hspace{0.7em}
\includegraphics[width=.45\linewidth]{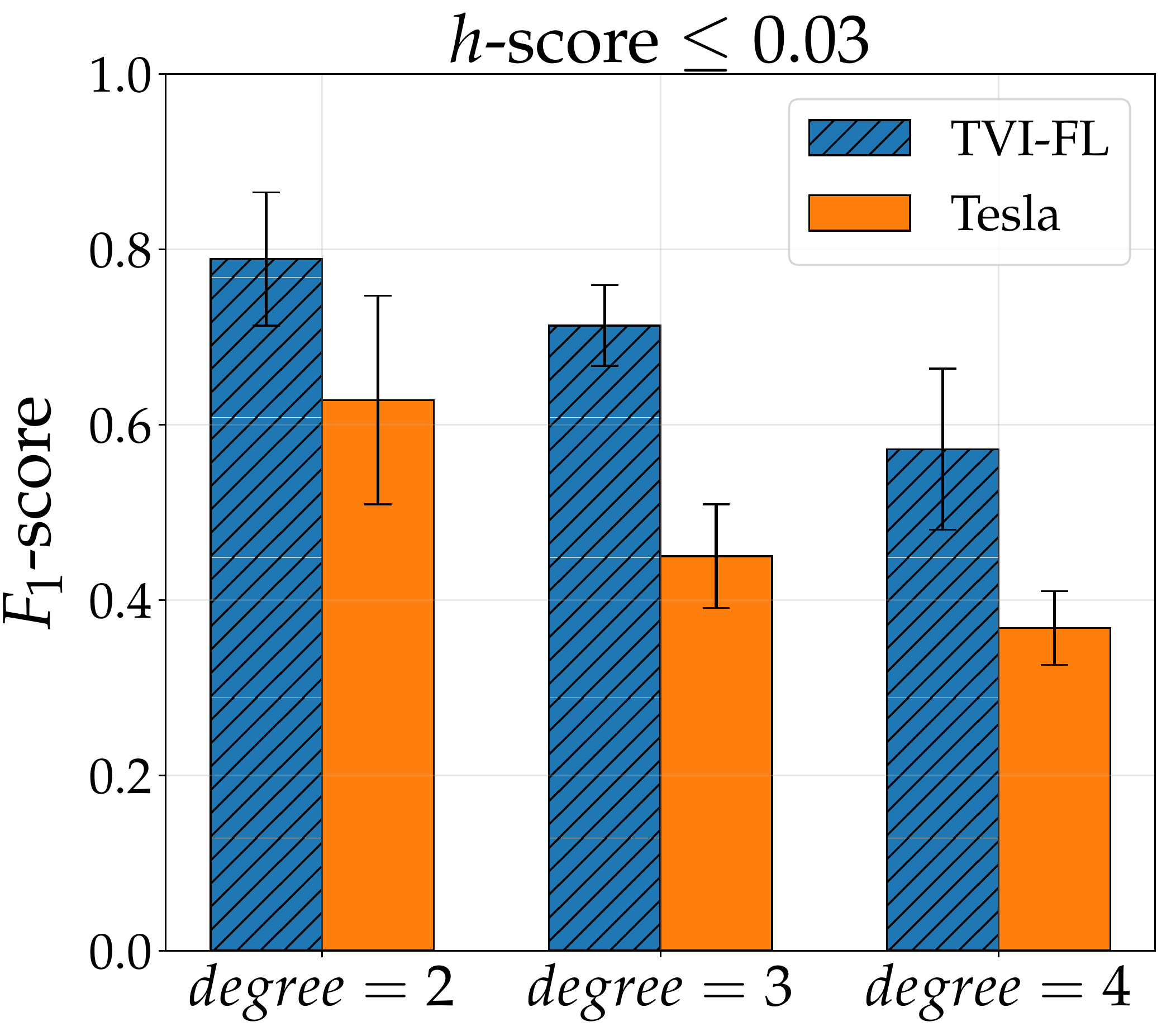}
\vspace{-0.5em}
\caption{The average value of the best $F_1$-score obtained when the $h$-score is below a certain threshold. These thresholds are respectively (from top left to bottom right), $0.0$, $0.01$, $0.02$, $0.03$, i.e. at most $0$, $1$, $2$, or $3$ timestamps of offset error between an estimated and a real change-point. Each pairs of bars corresponds to different $d$-regular graphs, with $d\in\{2,3,4\}$. The error bars correspond to $\pm$\,(std).}
\label{fig:barplot}
\end{figure}

\subsection{Finding change-points in the real world}

\paragraph{Dataset and setup.} In this section we evaluate the empirical performance of the  \proposedMethod method in a real-world use case. In particular, we analyze the different votes of the Illinois House of Representatives during the period of the $114$-th and $115$-th US Congresses (2015-2019), which are available at {\small\texttt{voteview.com}} \cite{usvote}. The Illinois House of Representatives has $18$ seats (one per district), each one corresponds to a US Representative belonging to the Democratic or the Republican parties. A Representative may or may not get reelected at the end of a Congress, which affects if he/she will retain his/her seat in the new Congress. The specific dataset we used contains $1264$ votes, each of them represented by a vector of size $p=18$, where a dimension is equal to $1$ if the respective Representative of that seat has voted \emph{Yes}, and $-1$ if it has voted \emph{No}. When no information is provided about the vote of a seat (e.g. due to an absence), we impute the majority vote of its party.

It is always difficult to interpret a large number of change-points. For this reason, we choose to use the AIC criterion, which was found in Sec.\,\ref{seq:xp} to favor smaller number of change-points. As for model tuning, we use a grid-search strategy to find the best values for the hyperparameters.

\paragraph{Results.} 

Fig.\,\ref{fig:Us_cum_vote} (bottom) shows the cumulative function of the votes of each of the $18$ seats, in temporal order, and the three change-points (dashed vertical line) detected by \proposedMethod.
The first two change-points are difficult to interpret; it seems though that the second one corresponds to the pre-election period when a Congress comes to its end and votes get usually less polarized. Nevertheless, it must be noted that the structural changes of the first two change-points are significantly lower compared to the third one. In fact, this last estimated change-point corresponds exactly to the time at which the Congress has changed. This significant change-point seems due to the non-reelection of some Representatives. More specifically, the Representative of $10$-th seat was the only one who was not reelected at the end of the $114$-th Congress: the Republican Robert Dold, who was replaced by the Democrat Brad Schneider. This switch apparently lead to a significant variation in the structure of the underlying graph. Fig.\,\ref{fig:Us_cum_vote} (top) %
shows the graphs of positive weights, before and after this significant change-point. As expected, two clusters appear, one with the seats of Democrats and the other with those of the Republicans. Moreover, the $10$-th seat becomes more connected with the cluster of Democrats after the time of change: the node loses $3$ connections to the Republican cluster and gains $5$ connections to Democrats and gets connected with all of them. More generally, all its weights with the Republican cluster decreases, contrarily to its weights with the Democratic cluster that do increase. This observation explains the origins of the structural change. Finally, it is interesting to observe that before and after the change-point, the $10$-th seat is the only one well-connected to both political groups. This makes us to conclude that this seat is represented by a \emph{super-collaborator}, a role that some Representatives get by acting more independently and position themselves in the middle of the parties \cite{andris2015rise}. Similarly, it is not surprising for Dan Lipinski, who had the $3$-rd seat, to present in the learned graphs $2$ connections with Republicans, as he is known to be a conservative Democrat.

Overall, this experiment shows that \proposedMethod is suited to find change-points in a real-world binary dataset, while also to recover the underlying evolving graph structure. This way, it increases the interpretability of the detected change-points. After applying the Telsa method on the same problem, we observed similar results and for this reason we omit them from the presentation. %

\begin{figure}[t]
\includegraphics[width=.48\linewidth,viewport=44pt 26pt 390pt 280pt, clip]{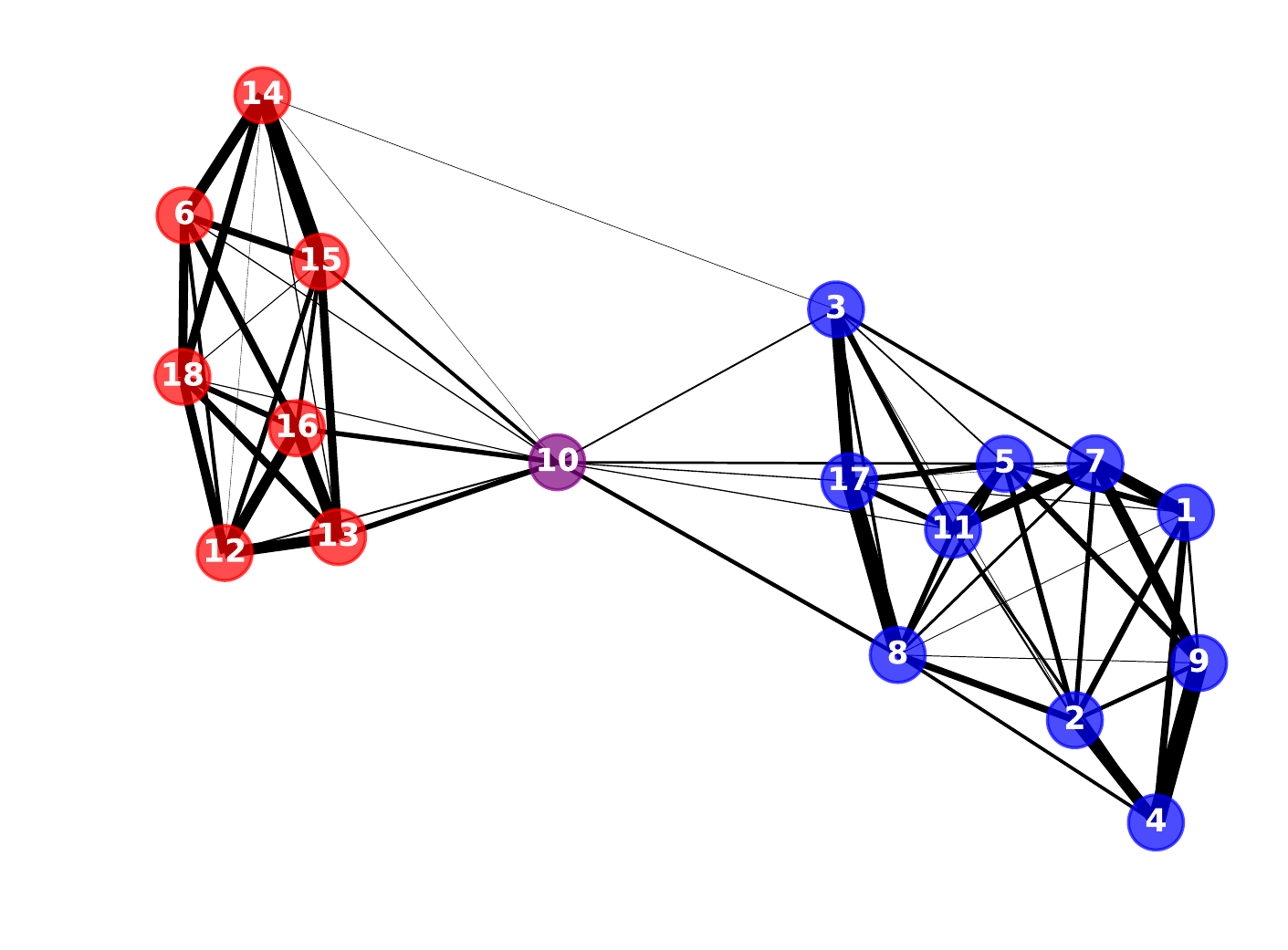}
\hspace{0.4em}
\includegraphics[width=.48\linewidth,viewport=44pt 26pt 390pt 280pt, clip]{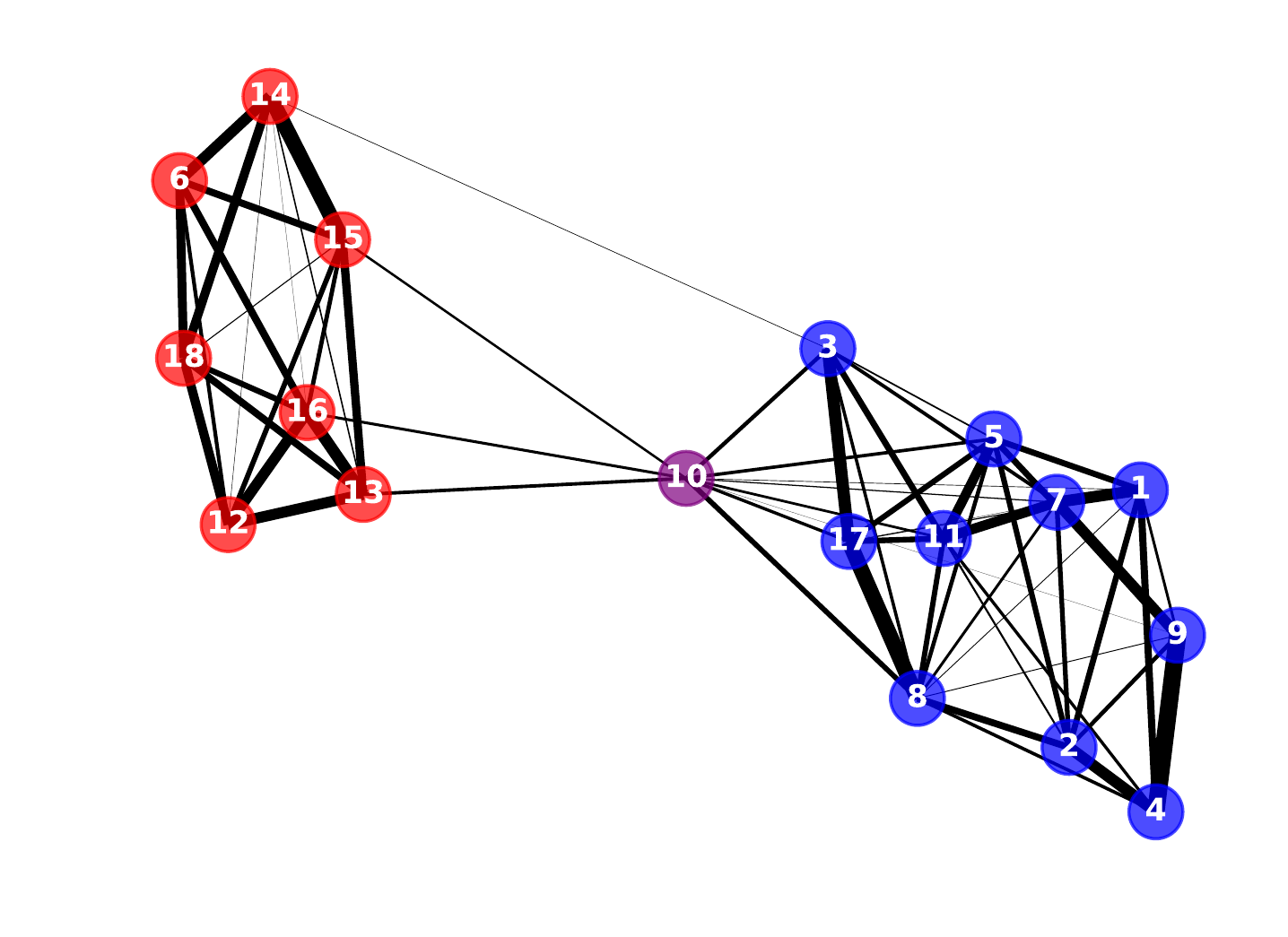}\!\!
\vspace{-0.9em}
\hspace{0.0em} \includegraphics[width=0.9\linewidth,viewport=0pt 0pt 418pt 310pt, clip]{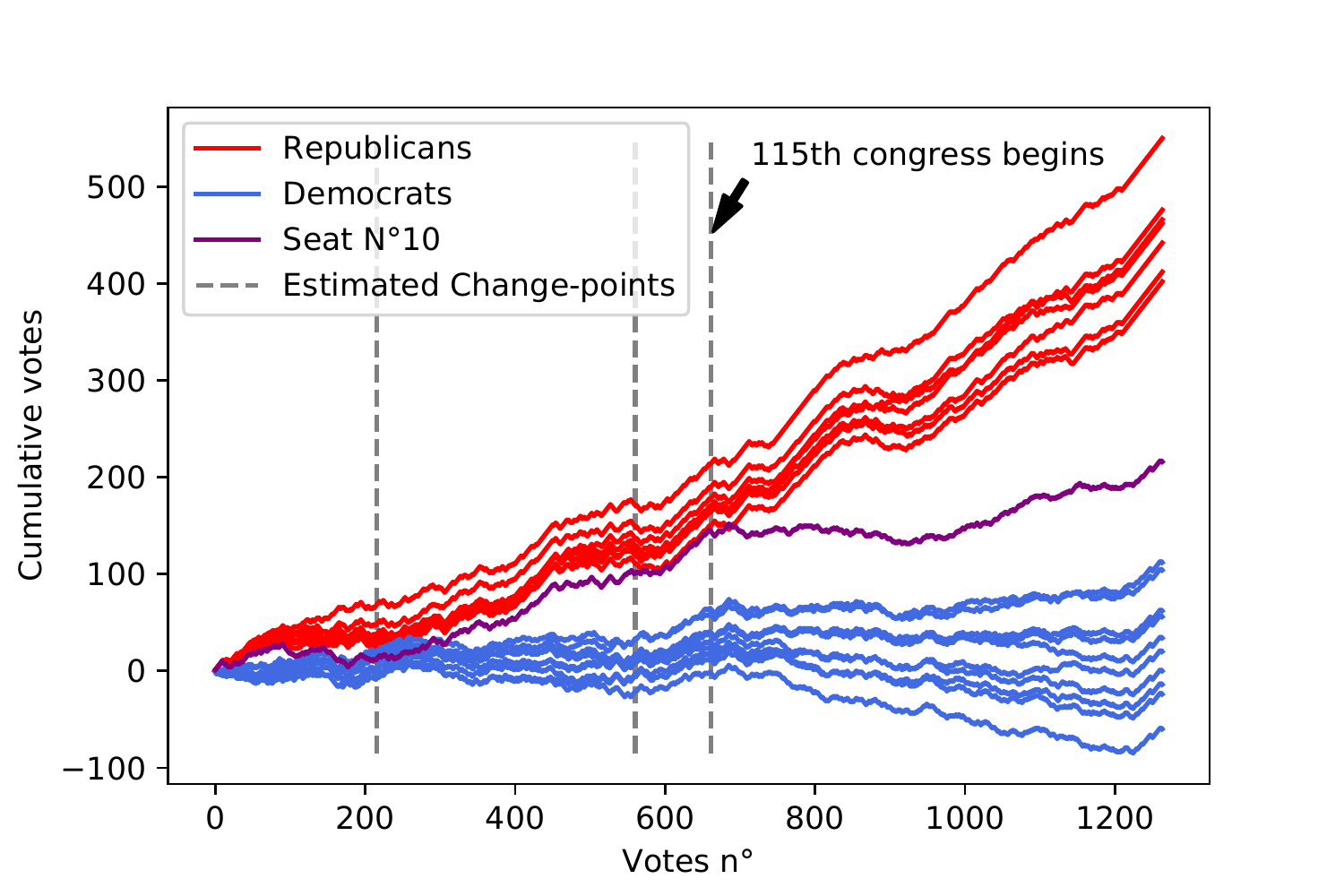}
\caption{(Top) The two graphs before and after the strongest estimated change-point: the third one that indeed corresponds to the end of the $114$-th Congress. (Bottom) The cumulative functions of votes of the $18$ seats over the two Congresses.}
\label{fig:Us_cum_vote}
 \end{figure}

\section{Conclusion and Future work}

This paper proposed \proposedMethod, an efficient way to learn a time-varying Ising model with piece-wise constantly evolving structure. Our method is able to both detect the change-points at which the structure of the model changes and the structure themselves. Our work is the first to provide change-point consistency theorems in this context. Those theoretical guarantees are reinforced by an empirical study. Using two different model selection criteria, the proposed method is showed to outperform the closest baseline algorithm. Directions of future research may include the investigation of a more adapted and refined way to solve the optimization program,  the proof of consistent graph structure recovery (\emph{sparsistency}) or the use of the recent Interaction Screening Objective \cite{vuffray2016interaction} in place of the standard conditional likelihood.

\section*{Acknowledgments}
This work was funded by the IdAML Chair hosted at ENS Paris-Saclay, Universit\'e Paris-Saclay.

\balance
\bibliographystyle{mystyle}
{
\footnotesize
\bibliography{biblio}
}

\end{document}


\twocolumn[

\icmltitle{Appendix}

\vskip 0.3in
]

\section*{Additional figures and results}

\subsection*{Comparison between $\boldsymbol{\ell_2}$- and $\boldsymbol{\ell_1}$-norms}

In Fig.\,\ref{fig:l1vsl2}, we illustrate the main difference in using an $\ell_2$- or alternatively a $\ell_1$-norm in the fused penalty of our objective function. The figure illustrates well the problem of $\ell_1$-norm: by penalizing each dimension independently, this norm easily leads to parameter vectors that have some non-zero dimensions, making the piece-wise constant assumption more difficult to recover. On the contrary, the $\ell_2$-norm avoids this problem and hence enforces the whole consecutive parameter vectors to be equal. 

\subsection*{Another real-world experiment}

In this section, we evaluate the goodness of graph learning with TVI-FL on the Sigfox IoT dataset \citep{le2019probabilistic} (available at: \texttt{http:/\!/kalogeratos.com/the-sigfox-iot-dataset}). The dataset contains activity recorded on a telecommunication network, where each observation corresponds to a message that was locally broadcasted by one device and has been received by a subset of the $34$ monitored antennas. Each data vector is binary and indicates which antennas has received the message or not ($\text{received}=1$, not $\text{received}=0$). The dataset contains all the messages received by the antennas, on a daily basis over a period of five months, resulting in $n=120$ timestamps. According to the authors, one antenna is working poorly after the $30$-th timestamp. In the following experiment, we select this antenna along with the $19$ geographically closest others, and we select randomly $n_i = 200$ messages at each timestamp. %
The learned graphs with TVI-FL at timestamps $i=0$ (before the antenna's malfunction) and $i = 60$ (after the antenna's malfunction) are displayed in Fig.~\ref{fig:sigfox}, where only the positive edges are drawn.

\begin{figure}[t]
  \centering
  \includegraphics[width=0.9\linewidth]{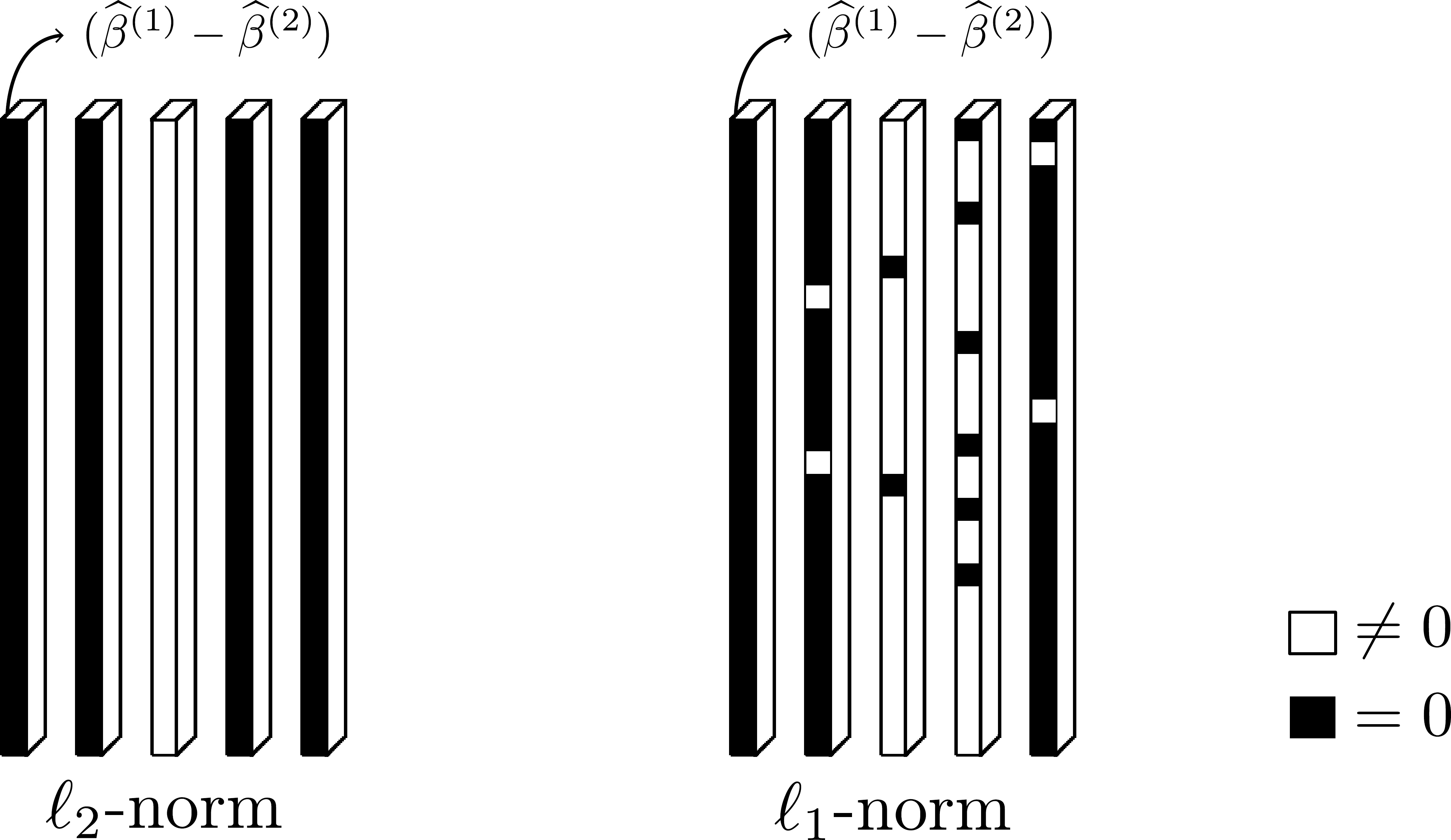}
  \caption{Comparison of the learned parameter vectors when using either $\ell_2$- or $\ell_1$-norm in the fused penalty. White squares indicates dimensions at which the two consecutive parameter vectors are different. Black squares where they are equal. The presence of at least one white square indicates a change-point.}
\label{fig:l1vsl2}
 \end{figure}
 \begin{figure}[t]
  \centering
  \includegraphics[width=1.1\linewidth]{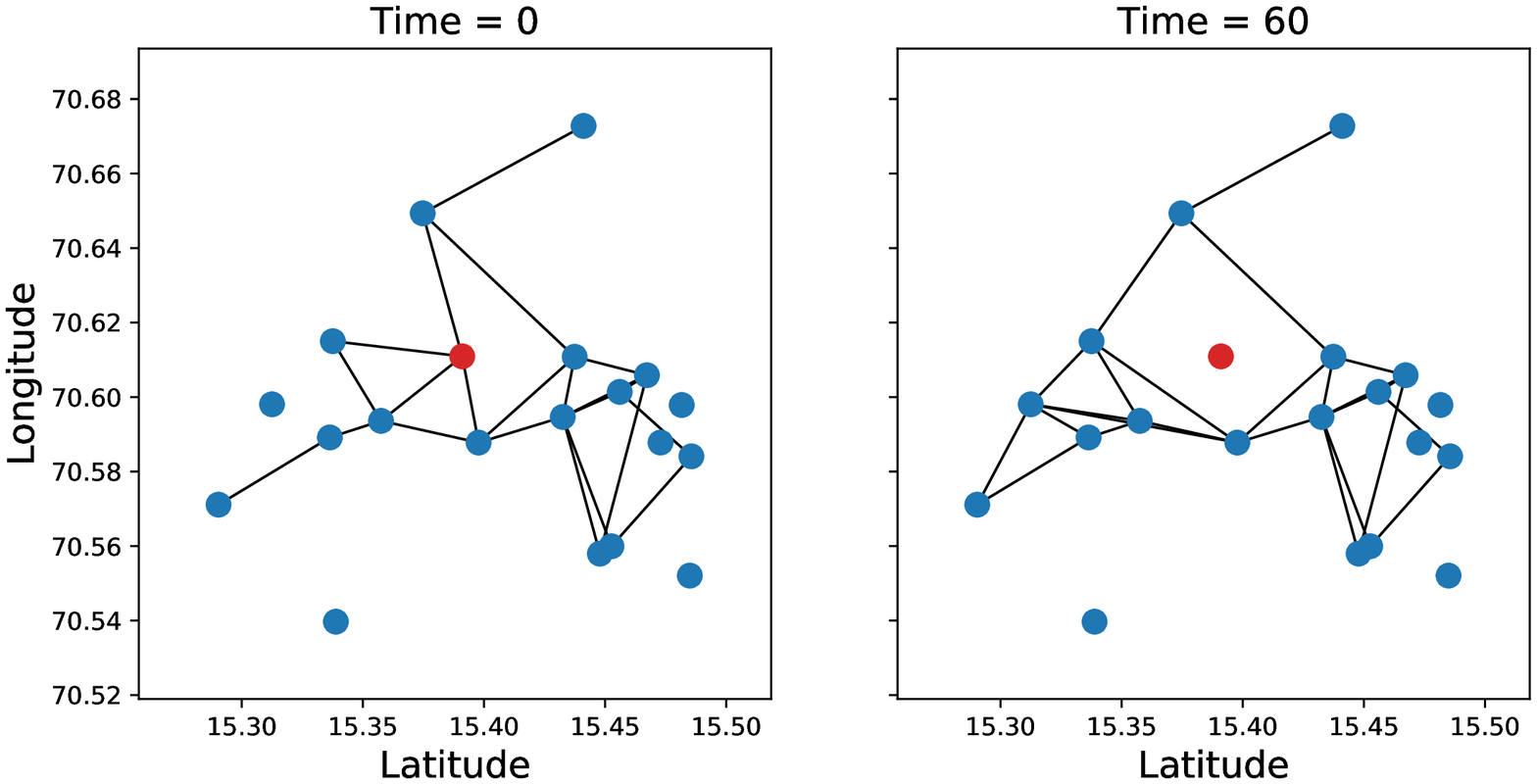}
  \caption{Learned graphs for Sigfox dataset before and after the anomaly  recorded at the red antenna.}
\label{fig:sigfox}
 \end{figure}

The goodness-of-fit of our method can be corroborated by the observations: 1) The learned graphs are 
in agreement with the spatial distribution of the antennas: nearby antennas are more likely to be connected as they have high chance to receive the same messages; 2) The problematic antenna lost edges after its malfunction. Again, this is as expected since a poorly working antenna would receive less messages, implying a decreased correlation with its neighbors.

\subsection*{A more complete table of results}

For completeness, in Tab.\,\ref{table:res_table} we complete the table of the main text of the paper with an additional comparative method, namely the one that estimates a graph at each timestamp ($\lambda_1 = 0$).

\begin{table*}[t]\scriptsize%
    \centering
	\ra{1.0}
	\begin{tabular}{@{}lll|llr|lll@{}}
		\toprule
         \textbf{} & \textbf{Observations} & & \multicolumn{3}{c}{\textbf{AIC}} & \multicolumn{3}{c}{\textbf{AUC}} \\
		\textbf{Degree} & \textbf{per timestamp} & \textbf{Method}  & \textbf{$h$-score} $\downarrow$ & \textbf{$F_1$-score} $\uparrow$ & $\boldMath{\widehat{D}\hfill}$ & \textbf{$h$-score} $\downarrow$ & \textbf{$F_1$-score} $\uparrow$ & $\boldMath{\widehat{D}\hfill}$\\ \midrule
		%
		\multirow{2}{*}{$d=2$} & \multirow{2}{*}{$n^{(i)}=4$} & TVI-FL   & $\boldMath{0.046 \pm (0.024)}$ & $\boldMath{0.694 \pm (0.103)}$ & $7.400 \pm (3.137)$ & $0.221 \pm (0.035)$ & $\boldMath{0.876 \pm (0.030)}$ & $26.100 \pm (7.739)$ \\
		&& Tesla & $0.106 \pm (0.087)$ & $0.649 \pm (0.190)$ & $12.700 \pm (7.682)$ & $\boldMath{0.184 \pm (0.051)}$ & $0.841 \pm (0.041)$ & $25.100 \pm (4.784)$\\
		&& $\lambda_1 = 0$ & $0.290 \pm (0.000)$ & $0.342 \pm (0.007)$ & $99.000 \pm (0.000)$ & $0.290 \pm (0.000)$ & $0.342 \pm (0.000)$ & $99.000 \pm (0.000)$\\
		\cmidrule{2-9}
		\multirow{2}{*}{} & \multirow{2}{*}{$n^{(i)}=6$} & TVI-FL   &  $\boldMath{0.129 \pm (0.058)}$ & $\boldMath{0.816 \pm (0.073)}$ & $9.700 \pm (2.759)$ & $\boldMath{0.147 \pm (0.071)}$ & $\boldMath{0.875 \pm (0.027)}$ & $15.300 \pm (3.378)$ \\
		&& Tesla &  $0.178 \pm (0.130)$ & $0.748 \pm (0.167)$ & $12.900 \pm (5.540)$ & $0.164 \pm (0.062)$ & $0.841 \pm (0.048)$ & $19.000 \pm (2.530)$ \\
		&& $\lambda_1 = 0$ & $0.290 \pm (0.000)$ & $0.407 \pm (0.010)$ & $99.000 \pm (0.000)$ & $0.290 \pm (0.000)$ & $0.407 \pm (0.000)$ & $99.000 \pm (0.000)$\\
		\cmidrule{2-9}
		\multirow{2}{*}{} & \multirow{2}{*}{$n^{(i)}=8$} & TVI-FL   &  $\boldMath{0.082 \pm (0.081)}$ & $0.833 \pm (0.095)$ & $7.400 \pm (3.040)$ & $\boldMath{0.099 \pm (0.073)}$ & $\boldMath{0.891 \pm (0.024)}$ & $11.000 \pm (3.873)$ \\
		&& Tesla &  $0.124 \pm (0.071)$ & $\boldMath{0.846 \pm (0.047)}$ & $13.600 \pm (2.010)$ & $0.178 \pm (0.066)$ & $0.853 \pm (0.039)$ & $14.700 \pm (3.348)$ \\
		&& $\lambda_1 = 0$ & $0.290 \pm (0.000)$ & $0.449 \pm (0.009)$ & $99.000 \pm (0.000)$ & $0.290 \pm (0.000)$ & $0.449 \pm (0.000)$ & $99.000 \pm (0.000)$\\
		\midrule
		%
		\multirow{2}{*}{$d=3$} & \multirow{2}{*}{$n^{(i)}=4$} & TVI-FL   & $\boldMath{0.080 \pm (0.069)}$ & $\boldMath{0.563 \pm (0.089)}$ & $7.000 \pm (2.683)$ & $\boldMath{0.204 \pm (0.035)}$ & $\boldMath{0.734 \pm (0.024)}$ & $23.100 \pm (6.715)$ \\
		&& Tesla & $0.278 \pm (0.319)$ & $0.353 \pm (0.072)$ & $3.200 \pm (2.891)$ & $0.208 \pm (0.029)$ & $0.611 \pm (0.041)$ & $29.200 \pm (3.187)$ \\
		&& $\lambda_1 = 0$ & $0.290 \pm (0.000)$ & $0.366 \pm (0.010)$ & $99.000 \pm (0.000)$ & $0.290 \pm (0.000)$ & $0.366 \pm (0.000)$ & $99.000 \pm (0.000)$\\
		\cmidrule{2-9}
		\multirow{2}{*}{} & \multirow{2}{*}{$n^{(i)}=6$} & TVI-FL   & $\boldMath{0.055 \pm (0.064)}$ & $\boldMath{0.617 \pm (0.161)}$ & $6.300 \pm (3.494)$ & $\boldMath{0.130 \pm (0.051)}$ & $\boldMath{0.743 \pm (0.034)}$ & $12.800 \pm (2.821)$ \\
		&& Tesla & $0.302 \pm (0.241)$ & $0.346 \pm (0.060)$ & $2.000 \pm (1.183)$ & $0.173 \pm (0.044)$ & $0.616 \pm (0.041)$ & $22.600 \pm (2.245)$ \\
		&& $\lambda_1 = 0$ & $0.290 \pm (0.000)$ & $0.391 \pm (0.014)$ & $99.000 \pm (0.000)$ & $0.290 \pm (0.000)$ & $0.391 \pm (0.000)$ & $99.000 \pm (0.000)$\\
		\cmidrule{2-9}
		\multirow{2}{*}{} & \multirow{2}{*}{$n^{(i)}=8$} & TVI-FL   & $\boldMath{0.091 \pm (0.073)}$ & $\boldMath{0.714 \pm (0.130)}$ & $8.000 \pm (2.530)$ & $\boldMath{0.127 \pm (0.073)}$ & $\boldMath{0.764 \pm (0.032)}$ & $10.400 \pm (2.154)$ \\
		&& Tesla & $0.311 \pm (0.231)$ & $0.361 \pm (0.098)$ & $2.600 \pm (2.615)$ & $0.162 \pm (0.052)$ & $0.633 \pm (0.045)$ & $18.700 \pm (3.716)$ \\
		&& $\lambda_1 = 0$ & $0.290 \pm (0.000)$ & $0.410 \pm (0.015)$ & $99.000 \pm (0.000)$ & $0.290 \pm (0.000)$ & $0.410 \pm (0.000)$ & $99.000 \pm (0.000)$\\
		\midrule
		%
		\multirow{2}{*}{$d=4$} & \multirow{2}{*}{$n^{(i)}=4$} & TVI-FL   & $\boldMath{0.101 \pm (0.082)}$ & $\boldMath{0.453 \pm (0.111)}$ & $6.500 \pm (3.324)$ & $\boldMath{0.232 \pm (0.026)}$ & $\boldMath{0.644 \pm (0.041)}$ & $29.400 \pm (4.317)$\\
		&& Tesla & $0.444 \pm (0.273)$ & $0.347 \pm (0.044)$ & $2.875 \pm (1.900)$ & $0.234 \pm (0.017)$ & $0.518 \pm (0.046)$ & $34.625 \pm (1.654)$\\
		&& $\lambda_1 = 0$ & $0.290 \pm (0.000)$ & $0.388 \pm (0.005)$ & $99.000 \pm (0.000)$ & $0.290 \pm (0.000)$ & $0.388 \pm (0.000)$ & $99.000 \pm (0.000)$\\
		\cmidrule{2-9}
		\multirow{2}{*}{} & \multirow{2}{*}{$n^{(i)}=6$} & TVI-FL  & $\boldMath{0.099 \pm (0.064)}$ & $\boldMath{0.501 \pm (0.130)}$ & $5.667 \pm (2.309)$ & $\boldMath{0.183 \pm (0.044)}$ & $\boldMath{0.664 \pm (0.041)}$ & $16.778 \pm (3.258)$ \\
		&& Tesla & $0.258 \pm (0.236)$ & $0.355 \pm (0.035)$ & $2.500 \pm (1.118)$ & $0.215 \pm (0.032)$ & $0.503 \pm (0.040)$ & $26.000 \pm (4.472)$ \\
		&& Static & $0.290 \pm (0.000)$ & $0.390 \pm (0.007)$ & $99.000 \pm (0.000)$ & $0.290 \pm (0.000)$ & $0.390 \pm (0.000)$ & $99.000 \pm (0.000)$\\
		\cmidrule{2-9}
		\multirow{2}{*}{} & \multirow{2}{*}{$n^{(i)}=8$} & TVI-FL  & $\boldMath{0.077 \pm (0.076)}$ & $\boldMath{0.528 \pm (0.158)}$ & $5.556 \pm (3.624)$ & $\boldMath{0.169 \pm (0.064)}$ & $\boldMath{0.678 \pm (0.049)}$ & $12.444 \pm (4.524)$ \\
		&& Tesla & $0.251 \pm (0.230)$ & $0.357 \pm (0.044)$ & $2.625 \pm (0.696)$ & $0.219 \pm (0.027)$ & $0.518 \pm (0.054)$ & $24.000 \pm (2.398)$ \\
		&& $\lambda_1 = 0$ & $0.290 \pm (0.000)$ & $0.385 \pm (0.007)$ & $99.000 \pm (0.000)$ & $0.290 \pm (0.000)$ & $0.385 \pm (0.000)$ & $99.000 \pm (0.000)$ \\
		\bottomrule
	\end{tabular}
	\caption{Results for the model with the lowest AIC, and that with the highest AUC. The average $\pm$\,(std) of the metrics is reported. Compared to the table provided in the main text, here an additional comparative method is mentioned, namely the one that estimates a graph at each timestamp ($\lambda_1 = 0$).}
    \label{table:res_table}
\end{table*}

\section*{Technical proofs}
\allowdisplaybreaks

\subsection*{Main results}
In the following, we recall and prove the main results given in the paper. The proofs uses in many situations the different lemmas given next.

\begin{lemma}{(Optimality Conditions)}
\label{lemma:optim_cond}
A matrix $\hat{\beta}$ is optimal for our problem iff there exists a collection of subgradient vectors $\{\hat{z}^{(i)}\}_{i=2}^n$ and $\{\hat{y}^{(i)}\}_{i=1}^n$, with $\hat{z}^{(i)} \in \partial \norm{\widehat{\beta}^{(i)}-\widehat{\beta}^{(i-1)}}_2$ and $\hat{y}^{(i)} \in \partial \norm{\widehat{\beta}^{(i)}}_1$, such that $\forall k=1,\ldots,n$ we have: 
\begin{align}
\sum_{i=k}^n & x^{(i)}_{\smallsetminus a}\left\{\tanh\!\left( \widehat{\beta}^{(i)\top}x^{(i)}_{\smallsetminus a})\right) - \tanh\!\left( \omega_a^{(i)\top}x^{(i)}_{\smallsetminus a})\right)\right\} \nonumber \\
& - \sum_{i=k}^n x^{(i)}_{\smallsetminus a}\left\{x^{(i)}_a - \mathbb{E}_{\Omega^{(i)}}\left[X_a | X_{\smallsetminus a} = x^{(i)}_{\smallsetminus a} \right]\right\} \nonumber \\ 
& + \lambda_1\hat{z}^{(k)} + \lambda_2\sum_{i=k}^n\hat{y}^{(i)} = \bold{0}_{p-1},
\end{align}
\text{where} $\tanh$ is the hyperbolic tangent function, $\bold{0}_{p-1}$ is the %
zero vector of size $p-1$, $\hat{z}^{(1)}=\bold{0}_{p-1}$, and \\
$\hat{z}^{(i)} = \left\{\begin{array}{l}
\frac{\widehat{\beta}^{(i)}-\widehat{\beta}^{(i-1)}}{\norm{\widehat{\beta}^{(i)}-\widehat{\beta}^{(i-1)}}_2}\quad \text{ if \ } \widehat{\beta}^{(i)}-\widehat{\beta}^{(i-1)} \neq 0, \\ 
\in \mathcal{B}_2(0,1) \quad \quad\ \ \,\text{otherwise};
\end{array} \right.$ \\ $\hat{y}^{(i)} = \left\{\begin{array}{l} \text{\emph{sign}}(\widehat{\beta}^{(i)}) \quad\quad\quad \text{if \ } x \neq 0, \\
\in \mathcal{B}_1(0,1) \quad\quad\ \ \  \text{otherwise}.\end{array}\right.$
\end{lemma}
\begin{proof}
Let us first introduce the following change of variables:
\begin{equation*}
    \gamma^{(i)} = \left\{\begin{array}{l}\beta^{(i)} \quad \quad \quad \quad \quad \text{if} \quad i=1, \\ 
    \beta^{(i)} - \beta^{(i-1)} \quad \text{otherwise.}
    \end{array} \right.
\end{equation*}
Thus $\beta^{(i)} = \sum_{l=1}^i \gamma^{(l)}$, which leads to a change in the objective function (4) of the main paper:
\begin{align}
& \{\hat{\gamma}^{(i)}\}_{i=1}^n = \argmin_{\gamma \in  \mathbb{R}^{p-1\times n}} \sum_{i=1}^n\log\Bigg\{\exp\left( \sum_{l=1}^i \gamma^{(l)\top }x^{(i)}_{\smallsetminus a}\right) \nonumber \\
&\quad\quad\quad\quad\quad\quad\quad\quad\quad\quad + \exp\left( -\sum_{l=1}^i \gamma^{(l)\top }x^{(i)}_{\smallsetminus a}\right)\Bigg\} \nonumber \\
& \quad\quad\quad\quad\quad - \sum_{i=1}^n x_a^{(i)}\sum_{l=1}^i \gamma^{(l)\top }x^{(i)}_{\smallsetminus a} + \lambda_1\sum_{i = 2}^{n} \norm{\gamma^{(i)}}_2  \nonumber \\ 
& \quad\quad\quad\quad\quad\quad\quad\quad\quad\quad\quad\quad + \lambda_2\sum_{i = 1}^{n} \norm{\sum_{l=1}^i \gamma^{(l)}}_1. \label{eq:optim_gamma}
\end{align}
This problem is convex, thus a necessary and sufficient condition for $\{\hat{\gamma}^{(i)}\}_{i=1}^n$ to be a solution is that for all $k=1,\ldots,n$, the $(p-1)$-dimensional zero-vector \textbf{0}, belongs to the subdifferential of (\ref{eq:optim_gamma}), taken with respect to $\gamma^{(k)}$:
\begin{align*}
    \textbf{0} \in \sum_{i=k}^n x^{(i)}_{\smallsetminus a} & \left(\tanh\left(\sum_{l=1}^i \hat{\gamma}^{(l)\top }x^{(i)}_{\smallsetminus a}\right) - x_a^{(i)}\right) \\
    & +\lambda_1\partial\norm{\hat{\gamma}^{(k)}}_2 + \lambda_2\sum_{i=k}^n \partial\norm{\sum_{l=1}^i \hat{\gamma}^{(l)}}_1.
\end{align*}
Recall that \\
$\partial \norm{x}_2= \left\{\begin{array}{l} \left\{\frac{x}{\norm{x}_2}\right\} \quad \text{if} \quad x \neq 0 \\
\mathcal{B}_2(0,1) \quad \text{otherwise;}\end{array}\right.$\\ $\partial \norm{x}_1= \left\{\begin{array}{l} \left\{\text{sign}(x)\right\} \quad \!\!\!\!\!\text{if} \quad x \neq 0\\
\mathcal{B}_1(0,1) \quad \text{otherwise}\end{array}\right.$\!\!. 

Reapplying the change of variable, we obtain:
\begin{equation*}
     \textbf{0} = \sum_{i=k}^n x^{(i)}_{\smallsetminus a}\left(\tanh\left( \hat{\beta}^{(i)\top }x^{(i)}_{\smallsetminus a}\right) - x_a^{(i)}\right) +\lambda_1\hat{z}^{(k)} + \lambda_2\sum_{i=k}^n \hat{y}^{(i)}.
\end{equation*}
Noting that $\mathbb{E}_{\Omega^{(i)}}\left[X_a | X_{\smallsetminus a} = x^{(i)}_{\smallsetminus a} \right] = \tanh\left( \omega_a^{(i)\top }x^{(i)}_{\smallsetminus a})\right)$, we obtain the final result.
\end{proof}

\begin{theorem}{(Change-point consistency)} Let $\{x_i\}_{i=1}^n$ be a sequence of observations drawn from the piece-wise constant Ising model presented in Sec.\,$2$. Suppose (A1-A3) hold, and assume that $\lambda_1 \asymp \lambda_2 = \mathcal{O}(\sqrt{\log(n)/n})$. Let $\{\delta_n\}_{n\geq 1}$ be a non-increasing sequence that converges to $0$, and such that $\forall n>0$, $\Delta_{\min}\geq n\delta_n$, with $n\delta_n \rightarrow +\infty$. Assume further that $(i)$ $\frac{\lambda_1}{n\delta_n\xi_{\min}}\rightarrow 0$, $(ii)$ $\frac{\sqrt{p-1}\lambda_2}{\xi_{\min}}\rightarrow 0$, and $(iii)$ $\frac{\sqrt{p\log(n)}}{\xi_{\min}\sqrt{n\delta_n}}\rightarrow 0$. Then, if the correct number of change-points are estimated, we have $\widehat{D}=D$ and:
\begin{equation}
    \mathbb{P}(\max_{j=1,\ldots,D}|\hat{T}_j-T_j|\leq n\delta_n)\underset{n\rightarrow \infty}{\longrightarrow}1 .
    \label{eq:consitency_thme}
\end{equation}
\label{thme}
\end{theorem}
\begin{proof}
The proof follows the steps given in \cite{harchaoui2010multiple,kolar2012estimating,gibberd2017multiple}. First of all, thanks to the union bound,
\begin{equation*}
    \mathbb{P}(\max_{j=1,\ldots,D }|\hat{T}_j-T_j|>n\delta_n)\leq \sum_{j=1}^{D } \mathbb{P}(|\hat{T}_j-T_j|>n\delta_n),
\end{equation*}
thus it suffices to show for each $j=1,\ldots, D $, that $\mathbb{P}(|\hat{T}_j-T_j|>n\delta_n)\rightarrow 0$. We denote by $A_{n,j}$ the event $\left\{ |\hat{T}_j-T_j|>n\delta_n \right\}$.\\

\noindent
Similarly to \cite{kolar2012estimating}, we first consider the good case where we assume that the event $C_n=\left\{|\hat{T}_j-T_j|<\frac{\Delta_{\min}}{2}\right\}$ occurs. \\

\noindent
\subsection*{Bounding the good case}

For each $j=1,\ldots, D $, we show that $\mathbb{P}(A_{n,j} \cap C_n) \longrightarrow 0$. In particular, we suppose that $\hat{T}_j\leq T_j$ as the proof for $\hat{T}_j\geq T_j$ will be the same by symmetry. \\
Applying Lemma \ref{lemma:optim_cond} with $k=\hat{T}_j$ and $k=T _j$, subtracting one with the other and applying the $\ell_2$-norm, we obtain: %
\begin{align*}
0 &= \Bigg\| \sum_{i=\hat{T}_j}^{T_j -1} x^{(i)}_{\smallsetminus a}\left\{\tanh\left( \widehat{\beta}^{(i)\top}x^{(i)}_{\smallsetminus a})\right) - \tanh\left( \omega_a^{(i)\top}x^{(i)}_{\smallsetminus a})\right)\right\} \nonumber \\
& \quad\quad  - \sum_{i=\hat{T}_j}^{T_j -1} x^{(i)}_{\smallsetminus a}\left\{x^{(i)}_a - \mathbb{E}_{\Omega^{(i)}}\left[X_a | X_{\smallsetminus a} = x^{(i)}_{\smallsetminus a} \right]\right\} \\
& \quad\quad\quad\quad\quad\quad\quad\quad + \lambda_1(\hat{z}^{(\widehat{T}_j)}-\hat{z}^{(T_j)}) + \lambda_2\sum_{i=\hat{T}_j}^{T_j -1}\hat{y}^{(i)}\Bigg\|_2\\
& \geq \Bigg\|\sum_{i=\hat{T}_j}^{T_j -1} x^{(i)}_{\smallsetminus a}\left\{\tanh\left( \widehat{\beta}^{(i)\top}x^{(i)}_{\smallsetminus a})\right) - \tanh\left( \omega_a^{(i)\top}x^{(i)}_{\smallsetminus a})\right)\right\} \nonumber \\
& \quad \quad - \sum_{i=\hat{T}_j}^{T_j -1} x^{(i)}_{\smallsetminus a}\left\{x^{(i)}_a - \mathbb{E}_{\Omega^{(i)}}\left[X_a | X_{\smallsetminus a} = x^{(i)}_{\smallsetminus a} \right]\right\}\Bigg\|_2 \\
& \quad\quad\quad\quad\quad - \Bigg\|\lambda_2\sum_{i=\hat{T}_j}^{T_j -1}\hat{y}^{(i)}\Bigg\|_2  - \Big\|\lambda_1(\hat{z}^{(\widehat{T}_j)}-\hat{z}^{(T_j)}) \Big\|_2.
\end{align*}
We have: \\$\Big\|\lambda_1(\hat{z}^{(\widehat{T}_j)}-\hat{z}^{(T_j)}) \Big\|_2\leq 2\lambda_1$ \  and\\ $\Bigg\|\lambda_2\sum_{i=\hat{T}_j}^{T_j-1}\hat{y}^{(i)}\Bigg\|_2 \leq (T_j - \hat{T}_j)\sqrt{p-1}\lambda_2$. \\Furthermore, one may notice that for all $i\in \left\{\hat{T}_j,\ldots,T_j -1\right\}$, $\widehat{\beta}^{(i)}=\hat{\theta}_a^{j+1}$ and $\omega_a^{(i)} = \theta_a^{j}$. Adding and subtracting $\tanh\left( (\theta_{a}^{j+1})^{\top} x^{(i)}_{\smallsetminus a})\right)$, then applying again the triangle inequality, leads to the following result:
\begin{align}
\label{eq:event_one}
\!\!\!\!\!\!\!2\lambda_1 + (T_j - \hat{T}_j)\sqrt{p-1}\lambda_2\geq \norm{R_1}_2 - \norm{R_2}_2 -\norm{R_3}_2\!\!\!
\end{align}
with 
\begin{align}
 R_1  & = \sum_{i=\widehat{T}_j}^{T_j-1}  x^{(i)}_{\smallsetminus a}\Big\{\tanh\Big( (\theta_{a}^{j})^{\top} x^{(i)}_{\smallsetminus a})\Big), \nonumber \\
  & \quad \quad \quad \quad \quad \quad \quad \quad \quad - \tanh\left( (\theta_{a}^{j+1})^{\top} x^{(i)}_{\smallsetminus a})\right)\Big\},  \\
R_2 & = \sum_{i=\widehat{T}_j}^{T_j-1}  x^{(i)}_{\smallsetminus a}\Big\{\tanh\left( (\hat{\theta}_{a}^{j+1})^{\top} x^{(i)}_{\smallsetminus a})\right), \nonumber \\ 
 &  \quad \quad \quad \quad \quad \quad \quad \quad \quad - \tanh\left( (\theta_{a}^{j+1})^{\top} x^{(i)}_{\smallsetminus a})\right)\Big\}, \\
R_3 & = \sum_{i=\widehat{T}_j}^{T_j-1}  x^{(i)}_{\smallsetminus a}\left\{x^{(i)}_a - \mathbb{E}_{\Theta^{(j)}}\left[X_a | X_{\smallsetminus a} = x^{(i)}_{\smallsetminus a} \right]\right\}.
\end{align}
The event (\ref{eq:event_one}) occurs with probability one and it can be showed that it is included in the event:
\begin{align*}
    \{2\lambda_1  &+ (T_j - \hat{T}_j)\sqrt{p-1}  \lambda_2\geq \frac{1}{3}\norm{R_1}_2\}  \\
    & \cup \{\norm{R_2}_2\geq \frac{1}{3}\norm{R_1}_2\} \cup \{\norm{R_3}_2\geq \frac{1}{3}\norm{R_1}_2\}.
\end{align*}
Thus, we have:
\begin{align*}
    & \mathbb{P}  (A_{n,j} \cap C_n)  \leq \\ 
    & \mathbb{P}(A_{n,j} \cap C_n \cap \{2\lambda_1 + (T_j - \hat{T}_j)\sqrt{p-1}\lambda_2\geq \frac{1}{3}\norm{R_1}_2\}) \\
    & \quad \quad + \mathbb{P}(A_{n,j} \cap C_n \cap \{\norm{R_2}_2\geq \frac{1}{3}\norm{R_1}_2\}) \\
    & \quad \quad + \mathbb{P}(A_{n,j} \cap C_n \cap \{\norm{R_3}_2\geq \frac{1}{3}\norm{R_1}_2\}) \\
    & \quad \quad \quad  \triangleq \mathbb{P}(A_{n,j,1}) + \mathbb{P}(A_{n,j,3}) + \mathbb{P}(A_{n,j,3}).
\end{align*}

Now, We are going to show that each one of the three events has a probability that converges to 0 as $n$ grows. Let's focus on $A_{n,j,1}$. Applying the mean-value theorem, we have for all $i=\widehat{T}_j,\ldots,T_j-1$: 
\begin{align}
    \label{eq:mvt}
    \tanh & \left( (\theta_{a}^{j})^{\top} x^{(i)}_{\smallsetminus a})\right) - \tanh\left( (\theta_{a}^{j+1})^{\top} x^{(i)}_{\smallsetminus a})\right) \nonumber \\
     & = (1-\tanh^2{(\Bar{\theta}^{iT}x^{(i)}_{\smallsetminus a})})x^{(i)\top}_{\smallsetminus a}(\theta_{a}^{j}-\theta_{a}^{j+1}),
\end{align}
with $\Bar{\theta}^i = \alpha^i\theta_{a}^{j}+(1-\alpha^i)\theta_{a}^{j+1}$, for a certain $\alpha^i \in [0,1]$. Combining (\ref{eq:mvt}) with the definition of $R_1$, we obtain:
\begin{align}
    & \norm{R_1}_2  = \Bigg\|\sum_{i=\widehat{T}_j}^{T_j-1}  x^{(i)}_{\smallsetminus a}\Big\{\tanh\left( (\theta_{a}^{j})^{\top} x^{(i)}_{\smallsetminus a})\right) \nonumber \\
    & \quad\quad \quad  \quad \quad  \quad \quad  \quad \quad  - \tanh\left( (\theta_{a}^{j+1})^{\top} x^{(i)}_{\smallsetminus a})\right)\Big\}\Bigg\|_2 \\
    & =  (T_j-\widehat{T}_j)\Bigg\|\frac{1}{T_j-\widehat{T}_j}\sum_{i=\widehat{T}_j}^{T_j-1} (1-\tanh^2{(\Bar{\theta}^{iT}x^{(i)}_{\smallsetminus a})}) \nonumber \\ 
    & \quad \quad  \quad \quad  \quad \quad  \quad \quad \times x^{(i)}_{\smallsetminus a}x^{(i)\top}_{\smallsetminus a}(\theta_{a}^{j}-\theta_{a}^{j+1})\Bigg\|_2 \\
    & \geq (T_j-\widehat{T}_j) \times \nonumber\\ 
    & \times \Lambda_{\min}\left(\frac{1}{T_j-\widehat{T}_j}\sum_{i=\widehat{T}_j}^{T_j-1} (1-\tanh^2{(\Bar{\theta}^{iT}x^{(i)}_{\smallsetminus a})}) x^{(i)}_{\smallsetminus a}x^{(i)\top}_{\smallsetminus a}\right) \nonumber \\
    & \quad \quad  \quad \quad  \times \norm{\theta_{a}^{j}-\theta_{a}^{j+1}}_2
\end{align}

Since, $\forall j$, $\norm{\theta_{a}^{j}}_2\leq M$ (A2), we have $\norm{\Bar{\theta}^i}_2 \leq M$ and $|\Bar{\theta}^{iT}x^{(i)}_{\smallsetminus a}| \leq M\cdot \sqrt{p-1}$. Thus, there exist a constant $\Tilde{M}>0$ such that $1-\tanh^2{(\Bar{\theta}^{iT}x^{(i)}_{\smallsetminus a})}\geq \Tilde{M}$. Combining this with the fact that each matrix $x^{(i)}_{\smallsetminus a}x^{(i)\top}_{\smallsetminus a}$ are positive semidefinite, we have:
\begin{align}
    & \norm{R_1}_2 \geq \nonumber \\
    & (T_j-\widehat{T}_j)\Tilde{M}\Lambda_{\min}\left(\frac{1}{T_j-\widehat{T}_j}\sum_{i=\widehat{T}_j}^{T_j-1} x^{(i)}_{\smallsetminus a}x^{(i)\top}_{\smallsetminus a}\right)\xi_{\min}. \label{eq:proof_R1}
\end{align}

Thus, the event $\{2\lambda_1 + (T_j - \hat{T}_j)\sqrt{p-1}\lambda_2\geq \frac{1}{3}\norm{R_1}_2\}$ is included in the event: 
\begin{align}
    \label{event:R1}
    2\lambda_1 & + (T_j - \hat{T}_j)\sqrt{p-1}\lambda_2 \geq \nonumber \\
    & (T_j-\widehat{T}_j)\Tilde{M}\Lambda_{\min}\left(\frac{1}{T_j-\widehat{T}_j}\sum_{i=\widehat{T}_j}^{T_j-1} x^{(i)}_{\smallsetminus a}x^{(i)\top}_{\smallsetminus a}\right)\xi_{\min}. 
\end{align}

Denoting by $\{\ref{event:R1}\}$ the event of Eq.\,\ref{event:R1}, we have:
\begin{align*}
    & \mathbb{P}(A_{n,j,1}) \leq \mathbb{P}(A_{n,j} \cap C_n \cap \{\ref{event:R1}\}) \\
    & \leq \mathbb{P}\Bigg(A_{n,j} \cap C_n \cap \{\ref{event:R1}\} \\
    &  \cap \{\Lambda_{\min}\left(\frac{1}{T_j-\widehat{T}_j}\sum_{i=\widehat{T}_j}^{T_j-1} x^{(i)}_{\smallsetminus a}x^{(i)\top}_{\smallsetminus a}\right) > \frac{\phi_{\min}}{2}\}\Bigg) \\
    &  + \mathbb{P}\Bigg(A_{n,j} \cap C_n \\
    & \cap\{\Lambda_{\min}\left(\frac{1}{T_j-\widehat{T}_j}\sum_{i=\widehat{T}_j}^{T_j-1} x^{(i)}_{\smallsetminus a}x^{(i)\top}_{\smallsetminus a}\right) \leq \frac{\phi_{\min}}{2}\}\Bigg).
\end{align*}

Using Lemma \ref{lemma:lambda_borne_rand} with $v_n=n\delta_n$ and $\epsilon=\frac{\phi_{\min}}{2}$, we can bound the right-hand side of the upper equation. We also re-write the first term so that we obtain:
\begin{align}
    & \mathbb{P}(A_{n,j,1}) \nonumber \\
    & \leq \mathbb{P}(A_{n,j} \cap C_n \cap \{\frac{2\lambda_1}{T_j-\widehat{T}_j} + \sqrt{p-1}\lambda_2 > \frac{\Tilde{M}\phi_{\min}}{2} \xi_{\min} \} ) \nonumber \\
    & \quad \quad \quad \quad + c_1\exp{\left(-\frac{\epsilon^2n\delta_n}{2} + 2\log(n)\right)} \nonumber \\
    & \leq \mathbb{P}(\xi_{\min}^{-1}\frac{2\lambda_1}{n\delta_n} + \xi_{\min}^{-1}\sqrt{p-1}\lambda_2 > \frac{\Tilde{M}\phi_{\min}}{2}   ) \nonumber \\
     & \quad \quad \quad \quad + c_1\exp{\left(-\frac{\epsilon^2n\delta_n}{2} + 2\log(n)\right)}. \nonumber
\end{align}

Thanks to $(iii)$, we have $n\delta_n$ that goes to infinity faster than $\log(n)$, thus the second term of the sum goes to $0$ as $n$ grows. Furthermore, using $(i)$ and $(ii)$ we have:
\begin{align*}
    \mathbb{P}(\xi_{\min}^{-1}\frac{2\lambda_1}{n\delta_n} +  & \xi_{\min}^{-1}\sqrt{p-1}\lambda_2 > \frac{\Tilde{M}\phi_{\min}}{2}) \\ 
    & \underset{n\rightarrow 0}{\longrightarrow} \mathbb{P}(0 + 0 > \frac{\Tilde{M}\phi_{\min}}{2})=0.
\end{align*}

\noindent
Which concludes that $\mathbb{P}(A_{n,j,1})\rightarrow 0$. \\

We now focus on the event $A_{n,j,2}$. Let $\Bar{T}_j \triangleq\lfloor 2^{-1}(T_{j}+T_{j+1}) \rfloor$ and remark that between $T_j$ and $\Bar{T}_j$, $\widehat{\beta}^{(i)} = \widehat{\theta}^{j+1}$. Now, using Lemma \ref{lemma:optim_cond} with $k=\Bar{T}_j$ and $k=T_j$ and similar operation used to show equation (\ref{eq:event_one}), we have:
\begin{align}
    & 2\lambda_1 +  (\bar{T}_j - T_j)\sqrt{p-1}\lambda_2 \nonumber \\ 
    & \geq \Bigg\|\sum_{i=T_j}^{\Bar{T}_j-1}  x^{(i)}_{\smallsetminus a}\bigg(\tanh\left((\hat{\theta}_{a}^{j+1})^{\top} x^{(i)}_{\smallsetminus a})\right) \nonumber  \\
     & \quad \quad \quad \quad \quad \quad \quad \quad  - \tanh\left( (\theta_{a}^{j+1})^{\top} x^{(i)}_{\smallsetminus a})\bigg)\right)\Bigg\|_2 \nonumber \\
    & - \norm{\sum_{i=T_j}^{\Bar{T}_j-1}  x^{(i)}_{\smallsetminus a}\underbrace{\left(x^{(i)}_a - \mathbb{E}_{\Theta^{(j+1)}}\left[X_a | X_{\smallsetminus a} = x^{(i)}_{\smallsetminus a} \right]\right)}_{\varepsilon^i_{j+1}}}_2. \nonumber
\end{align}

Now using the fact that $\norm{\hat{\theta}_{a}^{j+1}}_2$ is necessarily bounded, Lemma \ref{lemma:lambda_borne_rand} with $\epsilon = \phi_{\min}/2$ and similar arguments that we used for $A_{n,j,1}$, we can write that the first term in the right-hand side of the previous equation is lower-bounded by:
\begin{equation*}
    (T_j-\bar{T}_j)\Tilde{\Tilde{M}}\frac{\phi_{\min}}{2}\norm{\hat{\theta}_{a}^{j+1} - \theta_{a}^{j+1}}_2
\end{equation*}
with probability tending to one. Here, $\Tilde{\Tilde{M}}$ corresponds to a positive constant derived the same way as $\Tilde{M}$ in the previous part of the proof. In consequence, we can write 
\begin{align}
    & \norm{\hat{\theta}_{a}^{j+1} - \theta_{a}^{j+1}}_2 \leq \nonumber \\
    & \frac{8\lambda_1 + 4(\Bar{T}_j-T_j)\sqrt{p-1}\lambda_2 + 4\norm{\sum_{i=T_j}^{\Bar{T}_j-1}  x^{(i)}_{\smallsetminus a}\varepsilon^i_{j+1}}_2}{\Tilde{\Tilde{M}}\phi_{\min}(T_{j+1}-T_j)},
    \label{eq:Aj2}
\end{align}
which holds with probability tending to one. 

Furthermore, with probability also tending to one it can be shown using the same arguments used to prove equation (\ref{eq:proof_R1}) that $\norm{R_1}_2\geq (T_j-\widehat{T}_j)\Tilde{M}\phi_{\min}\xi_{\min}/2$ and $\norm{R_2}_2\leq \norm{\hat{\theta}_{a}^{j+1} - \theta_{a}^{j+1}}_2\phi_{\max}(T_j-\widehat{T}_j)/2$. Combining that with equation (\ref{eq:Aj2}), we can write:
\begin{align*}
    & \mathbb{P}(A_{n,j,2}) \\
    & \leq \mathbb{P}(A_{n,j} \cap C_n \cap \{ \frac{1}{3}\tilde{\tilde{M}}\tilde{M}\phi_{\min}^2\phi_{\max}^{-1}\xi_{\min}(T_{j+1}-T_j) \leq \\
     & \quad  8\lambda_1 + 4(\Bar{T}_j-T_j)\sqrt{p-1}\lambda_2 + 4\norm{\sum_{i=T_j}^{\Bar{T}_j-1}  x^{(i)}_{\smallsetminus a}\varepsilon^i_{j+1}}_2\}) \\
    & \quad \quad \quad \quad + c_1\exp{\left(-c_2n\delta_n + 2\log(n)\right)} \\
    & \leq \mathbb{P}(c_3\phi_{\min}^2\phi_{\max}^{-1}\xi_{\min}\Delta_{\min} \leq \lambda_1) \\ 
    & + \mathbb{P}(c_4\phi_{\min}^2\phi_{\max}^{-1}\xi_{\min} \leq \sqrt{p-1}\lambda_2) \\
    &   + \mathbb{P}\left(c_5\phi_{\min}^2\phi_{\max}^{-1}\xi_{\min} \leq (\Bar{T}_j-T_j)^{-1}\norm{\sum_{i=T_j}^{\Bar{T}_j-1}  x^{(i)}_{\smallsetminus a}\varepsilon^i_{j+1}}_2\right) \\
    &  + c_1\exp{\left(-c_2n\delta_n + 2\log(n)\right)}.
\end{align*}
With $c_1,\ldots,c_5$ positive constants.

The first two terms tends to $0$ as $n$ goes to infinity thanks to the hypothesis $(i)$ and $(ii)$ of the theorem. Indeed, since $\Delta_{\min}>n\delta_n$ and $(n\delta_n\xi_{\min})^{-1}\lambda_1\rightarrow 0$ $(i)$, the first term tends to $\mathbb{P}(c_3\phi_{\min}^2\phi_{\max}^{-1} \leq 0) = 0$ and the second term tends to $0$ since $\xi_{\min}^{-1}\sqrt{p-1}\lambda_2\rightarrow 0$ $(ii)$. The fourth term directly tends to $0$. Applying Lemma \ref{lemma:concentration_R3}, we can upper bound the third term by:
\begin{align*}
    \mathbb{P} & \left(c_5\phi_{\min}^2\phi_{\max}^{-1}\xi_{\min} \leq (\Bar{T}_j-T_j)^{-1/2}2\sqrt{p\log(n)}\right) \\
    & \quad + c_6\exp(-2p\log(n)) \\
     & \leq \mathbb{P} \left(c_5\phi_{\min}^2\phi_{\max}^{-1}\xi_{\min} \leq (n\delta_n)^{-1/2}2\sqrt{p\log(n)}\right)\\
     & \quad + c_6\exp(-2p\log(n))
\end{align*}
with $c_6$ an other positive constant.\\

Since $(\xi_{\min}\sqrt{n\delta_n})^{-1}\sqrt{p\log(n)}\rightarrow 0$ $(iii)$, the previous equation tends to $0$, which make $\mathbb{P}(A_{n,j,2})$ tends to $0$ as well. \\

Finally, we upper bound the probability on the event $A_{n,j,3}$. As before, we know that $\norm{R_1}_2\geq (T_j-\widehat{T}_j)\Tilde{M}\phi_{\min}\xi_{\min}/2$ with probability at least $1-c_1\exp(-c_2n\delta_n + 2\log(n))$, thus we have:
\begin{align*}
    \mathbb{P}(A_{n,j,3}) & \leq \mathbb{P}\left(\frac{\Tilde{M}\phi_{\min}\xi_{\min}}{6}\leq \frac{\norm{R_3}_2}{T_j-\widehat{T}_j}\right) \\ 
    & \quad + c_1\exp(-c_2n\delta_n + 2\log(n)).
\end{align*}

Using Lemma \ref{lemma:concentration_R3_random}, we can upper bound the first term by:
\begin{align*}
 & \mathbb{P}\left(\frac{\Tilde{M}\phi_{\min}\xi_{\min}}{6}\leq 2\sqrt{\frac{p\log(n)}{T_j-\widehat{T}_j}}\right) + c_2\exp(-c_3\log(n))  \\
 & \leq \mathbb{P}\left(\frac{\Tilde{M}\phi_{\min}\xi_{\min}}{6}\leq 2\sqrt{\frac{p\log(n)}{n\delta_n}}\right) + c_2\exp(-c_3\log(n)),\end{align*}
which tends to $0$ thanks to $(iii)$. Since the symmetric case follows exactly the same arguments, we have shown that $\mathbb{P}(A_{n,j}\cap C_n) \rightarrow 0$. We now need to prove that $\mathbb{P}(A_{n,j}\cap C_n^c) \rightarrow 0$.\\

\subsection*{Bounding the bad case} 

Let us define the following complementary events:
\begin{align}
    & D_n^{(l)} \triangleq \left\{\exists j \in [D ], \widehat{T}_j \leq T_{j-1} \right\} \cap C_n^c \\
    & D_n^{(m)} \triangleq \left\{\forall j \in [D ], T_{j-1} < \widehat{T}_j < T_{j+1} \right\} \cap C_n^c \\
    & D_n^{(r)} \triangleq \left\{\exists j \in [D ], \widehat{T}_j \geq T_{j+1} \right\} \cap C_n^c.
\end{align}
We can write $\mathbb{P}(A_{n,j}\cap C_n^c) = \mathbb{P}(A_{n,j}\cap D_n^{(l)}) +\mathbb{P}(A_{n,j}\cap D_n^{(m)}) +\mathbb{P}(A_{n,j}\cap D_n^{(r)})$. Again, the goal is to prove that the three terms tends to $0$. We will assume that $\widehat{T}_j \leq T_j$ as the other case can be done by symmetry. Let's first focus on the middle term, it has been shown in \cite{harchaoui2010multiple,kolar2012estimating,gibberd2017multiple} that it can be upper bounded in the following way:
\begin{align}
    & \mathbb{P}( A_{n,j}\cap D_n^{(m)}) \nonumber\\
    & \leq \mathbb{P} (A_{n,j}\cap \{(\widehat{T}_{j+1}-T_j)\geq \frac{\Delta_{\min}}{2}\}\cap D_n^{(m)}) \nonumber \\
    & \quad  + \mathbb{P}(\{(T_{j+1} - \widehat{T}_{j+1})\geq \frac{\Delta_{\min}}{2}\}\cap D_n^{(m)}) \nonumber \\
    & \leq \mathbb{P} (A_{n,j}\cap \{(\widehat{T}_{j+1}-T_j)\geq \frac{\Delta_{\min}}{2}\}\cap D_n^{(m)}) \nonumber \\
    & \quad  + \sum_{k=j+1}^{D }\mathbb{P}(\{(\widehat{T}_{k+1}-T_k)\geq \frac{\Delta_{\min}}{2}\} \nonumber \\
    & \quad \quad  \quad  \quad  \cap\{(T_{k}-\widehat{T}_k)\geq \frac{\Delta_{\min}}{2}\}\cap D_n^{(m)}).
    \label{eq:last_bound}
\end{align}

Let us bound the first term. Assuming the event $A_{n,j}\cap \{(\widehat{T}_{j+1}-T_j)\geq \frac{\Delta_{\min}}{2}\}\cap D_n^{(m)}$ and applying Lemma \ref{lemma:optim_cond} with $k=\widehat{T}_j$ and $k=T_j$, we can prove similarly as Eq.\,\ref{eq:Aj2} that:
\begin{align*}
     &\norm{\hat{\theta}_{a}^{j+1} - \theta_{a}^{j}}_2  \\ 
     & \leq \frac{4\lambda_1 + 2(T_{j}-\widehat{T}_j)\sqrt{p-1}\lambda_2 + 2\norm{\sum_{i=\widehat{T}_j}^{T_j-1}  x^{(i)}_{\smallsetminus a}\varepsilon^i_{j}}_2}{\tilde{\tilde{M}}\phi_{\min}(T_{j}-\widehat{T}_j)}\\
    & \leq c_1\phi_{\min}^{-1}(n\delta_n)^{-1}\lambda_1 + c_2\phi_{\min}^{-1}\sqrt{p-1}\lambda_2 \\ 
    &  \quad \quad  \quad  \quad + c_3 \phi_{\min}^{-1} (T_{j}-\widehat{T}_j)^{-1}\norm{\sum_{i=\widehat{T}_j}^{T_j-1}  x^{(i)}_{\smallsetminus a}\varepsilon^i_{j}}_2
\end{align*}
with probability tending to one. Using Lemma \ref{lemma:concentration_R3_random} we can bound the third term and obtain:
\begin{align*}
    \norm{\hat{\theta}_{a}^{j+1} - \theta_{a}^{j}}_2 & \leq c_1\phi_{\min}^{-1}(n\delta_n)^{-1}\lambda_1 + c_2\phi_{\min}^{-1}\sqrt{p-1}\lambda_2 \\ 
    &  + c_3 \phi_{\min}^{-1} (\sqrt{n\delta_n})^{-1}\sqrt{p\log(n)}
\end{align*}
with probability tending to one. Similarly, applying the same lemmas with $k=T_j$ and either $k=\widehat{T}_{j+1}$, if $\widehat{T}_{j+1}\leq T_{j+1}$ or $k=T_{j+1}$ otherwise, we have:
\begin{align*}
    \norm{\hat{\theta}_{a}^{j+1} - \theta_{a}^{j+1}}_2 & \leq c_4\phi_{\min}^{-1}(n\delta_n)^{-1}\lambda_1 + c_5\phi_{\min}^{-1}\sqrt{p-1}\lambda_2 \\
    & + c_6 \phi_{\min}^{-1} (\sqrt{n\delta_n})^{-1}\sqrt{p\log(n)}
\end{align*}
with probability tending to one. \\ 
Since $\xi_{\min} \leq \norm{\theta_{a}^{j} - \theta_{a}^{j+1}}_2 \leq  \norm{\hat{\theta}_{a}^{j+1} - \theta_{a}^{j}}_2 + \norm{\hat{\theta}_{a}^{j+1} - \theta_{a}^{j+1}}_2$, we finally upper bound the considered probability by:
\begin{align*}
    \mathbb{P} & (A_{n,j}\cap \{(\widehat{T}_{j+1}-T_j)\geq \frac{\Delta_{\min}}{2}\}\cap D_n^{(m)}) \\
     & \leq \mathbb{P}(\xi_{\min} \leq c_7\phi_{\min}^{-1}(n\delta_n)^{-1}\lambda_1 + c_8\phi_{\min}^{-1}\sqrt{p-1}\lambda_2 \\
     & + c_9 \phi_{\min}^{-1} (\sqrt{n\delta_n})^{-1}\sqrt{p\log(n)}).
\end{align*}
this tends to $0$ thanks to the hypothesis $(i)$, $(ii)$ and $(iii)$. The other probabilities in the upper bound on $\mathbb{P}( A_{n,j}\cap D_n^{(m)})$ also tends to $0$. The proof follows exactly the previous one. We proved that $\mathbb{P}( A_{n,j}\cap D_n^{(m)})\rightarrow 0$, we will now show the same for $\mathbb{P}(A_{n,j}\cap D_n^{(l)})$.

Following \cite{gibberd2017multiple}, we have:
\begin{align*}
    \mathbb{P}(D_n^{(l)}) & \leq \sum_{j=1}^{D }2^{j-1} \mathbb{P}(\max\{l \in [D ] : \widehat{T}_l\leq T_{l-1}\}) \\
    & \leq 2^{D -1}  \sum_{j=1}^{D }\sum_{l>j}\mathbb{P}(\{T_l - \widehat{T}_l \geq \frac{\Delta_{\min}}{2}\} \\ 
    & \quad \quad \quad \quad \cap \{\widehat{T}_{l+1}-T_l\geq \frac{\Delta_{\min}}{2}\}).
\end{align*}

Now, combining arguments of \cite{gibberd2017multiple} and those used to bound the elements of (\ref{eq:last_bound}), we have $\mathbb{P}(D_n^{(l)})\rightarrow 0$. Similarly we can show $\mathbb{P}(D_n^{(r)})\rightarrow 0$ as $n\rightarrow 0$. Finally we have $\mathbb{P}(A_{n,j}\cap C_n^c)\rightarrow 0$, which concludes the proof.
\end{proof}
\begin{proposition}
 Let $\{x_i\}_{i=1}^n$ be a sequence of observation drawn from the model presented in Sec.\, 2. Assume the condition of Theorem 1 are respected. Then, if for a fix $D_{\max}$ we have $D \leq \widehat{D}\leq D_{\max}$  then:
 \begin{equation*}
     \mathbb{P}(d(\mathcal{\widehat{D}}\| \mathcal{D})\leq n\delta_n)\underset{n\rightarrow \infty}{\longrightarrow}1.
 \end{equation*}
\end{proposition}
\begin{proof}
Let us show that:
\begin{align*}
     \mathbb{P} & (\{d(\mathcal{\widehat{D}}\| \mathcal{D})\geq n\delta_n\}\cap\{D \leq \widehat{D}\leq D_{\max}\}) \\ 
     & \leq \sum_{K = D}^{D_{\max}} \mathbb{P}(\{d(\mathcal{\widehat{D}}\| \mathcal{D})\geq n\delta_n\}\cap\{\widehat{D}=K\}) \underset{n\rightarrow \infty}{\longrightarrow}0.
\end{align*}
First, we note that for $K=D$, we have \\ $\mathbb{P}(\{d(\mathcal{\widehat{D}}\| \mathcal{D})\geq n\delta_n\}\cap\{\widehat{D}=K\}) \underset{n\rightarrow \infty}{\longrightarrow}0$ thanks to Theorem 1. Thus it suffices to show that:
\begin{align*}
     & \sum_{K = D+1}^{D_{\max}}  \mathbb{P}(\{d(\mathcal{\widehat{D}}\| \mathcal{D})\geq n\delta_n\}\cap\{\widehat{D}=K\}) \\
     & \leq \sum_{K = D+1}^{D_{\max}}\sum_{k = 1}^{D}\mathbb{P}(\forall1\leq l\leq K, |\widehat{T}_l-T_k|\geq n\delta_n) \underset{n\rightarrow \infty}{\longrightarrow}0.
\end{align*}
Like in \cite{harchaoui2010multiple}, we rewrite the event $\{\forall1\leq l\leq K, |\widehat{T}_l-T_k|\geq n\delta_n\}$ as the disjoint union of the events:
\begin{align*}
    & E_{n,k,1} = \{\forall1\leq l\leq K, |\widehat{T}_l-T_k|\geq n\delta_n \text{ and } \widehat{T}_l<T_k \} \\
    & E_{n,k,2} = \{\forall1\leq l\leq K, |\widehat{T}_l-T_k|\geq n\delta_n \text{ and } \widehat{T}_l>T_k \} \\
    & E_{n,k,3} = \{\exists 1\leq l\leq K-1, |\widehat{T}_l-T_k|\geq n\delta_n, \\
    & \quad \quad \quad \quad \quad   |\widehat{T}_{l+1}-T_k|\geq n\delta_n \text{ and } \widehat{T}_l<T_k<\widehat{T}_{l+1} \}
\end{align*}
and propose to show that the probability of each events tends to $0$ as $n$ grows.
Let's begin with $\mathbb{P}(E_{n,k,1})$ and note that it is equal to:
\begin{equation*}
    \mathbb{P}(E_{n,k,1}\cap \{\widehat{T}_{K}>T_{k-1}\}) + \mathbb{P}(E_{n,k,1}\cap \{\widehat{T}_{K}\leq T_{k-1}\})
\end{equation*}
First, we are going to upper bound the left-hand element of the previous equation. Applying Lemma \ref{lemma:optim_cond} with $t = \widehat{T}_K$ and $t = T_k$, we can prove similarly to the equation (\ref{eq:event_one}) in the good case scenario of the previous theorem that:
\begin{equation*}
2\lambda_1 + (T_k - \widehat{T}_K)\sqrt{p-1}\lambda_2\geq \norm{R'_1}_2 - \norm{R'_2}_2 -\norm{R'_3}_2
\end{equation*}
with
\begin{align*}
 R'_1  & = \sum_{i=\widehat{T}_K}^{T_k-1}  x^{(i)}_{\smallsetminus a}\Big\{\tanh\Big( (\theta_{a}^{k})^T x^{(i)}_{\smallsetminus a})\Big) \nonumber \\
  & \quad \quad \quad \quad \quad \quad \quad \quad \quad - \tanh\left( (\theta_{a}^{k+1})^T x^{(i)}_{\smallsetminus a})\right)\Big\}  \\
R'_2 & = \sum_{i=\widehat{T}_K}^{T_k-1}  x^{(i)}_{\smallsetminus a}\Big\{\tanh\left( (\hat{\theta}_{a}^{K+1})^T x^{(i)}_{\smallsetminus a})\right) \nonumber \\ 
 &  \quad \quad \quad \quad \quad \quad \quad \quad \quad - \tanh\left( (\theta_{a}^{k+1})^T x^{(i)}_{\smallsetminus a})\right)\Big\} \\
R'_3 & = \sum_{i=\widehat{T}_K}^{T_k-1}  x^{(i)}_{\smallsetminus a}\left\{x^{(i)}_a - \mathbb{E}_{\Theta^{(k)}}\left[X_a | X_{\smallsetminus a} = x^{(i)}_{\smallsetminus a} \right]\right\}.
\end{align*}
Like in the previous theorem, we can upperbound $\mathbb{P}(E_{n,k,1}\cap \{\widehat{T}_k>T_{k-1}\})$ by:
\begin{align*}
     \mathbb{P}(E^{(1)}_{n,k,1}) + \mathbb{P}(E^{(2)}_{n,k,1}) + \mathbb{P}(E^{(3)}_{n,k,1})
\end{align*}
where
\begin{align*}
    E^{(1)}_{n,k,1} &= \{2\lambda_1  + (T_k - \widehat{T}_K)\sqrt{p-1}  \lambda_2\geq \frac{1}{3}\norm{R'_1}_2\} \\
    E^{(2)}_{n,k,1} &= \{\norm{R'_2}_2\geq \frac{1}{3}\norm{R'_1}_2\} \\
    E^{(3)}_{n,k,1} &= \{\norm{R'_3}_2\geq \frac{1}{3}\norm{R'_1}_2\}.
\end{align*}
To show that $\mathbb{P}(E^{(1)}_{n,k,1})$ tends to $0$ it suffices to follow the proof used to show that $\mathbb{P}(A_{n,j,1})$ tends to $0$ in the good scenario of the previous theorem.

Similarly, to show that $\mathbb{P}(E^{(2)}_{n,k,1})$ tends to $0$ it suffices to follow the proof used for $\mathbb{P}(A_{n,j,2})$. Applying lemma \ref{lemma:optim_cond} with $t = T_k$ ans $t = T_{k+1}$ we can show that with probability tending to one:
\begin{align}
    & \norm{\widehat{\theta}_{a}^{K+1} - \theta_{a}^{k+1}}_2 \leq \nonumber \\
    & \frac{4\lambda_1 + 2(T_{k+1}-T_k)\sqrt{p-1}\lambda_2 + 2\norm{\sum_{i=T_k}^{T_{k+1}}  x^{(i)}_{\smallsetminus a}\varepsilon^i_{j+1}}_2}{\Tilde{\Tilde{M}}\phi_{\min}(T_{k+1}-T_k)}.
    \label{eq:Aj2}
\end{align}
The rest follows exactly the arguments used to show the limit of $\mathbb{P}(A_{n,j,2})$.

Finally, $\mathbb{P}(E^{(3)}_{n,k,1})$ tends to $0$ the same way $\mathbb{P}(A_{n,j,3})$ was tending to $0$ in the previous proof.

The proof to show that $\mathbb{P}(E_{n,k,1}\cap \{\widehat{T}_{K}\leq T_{k-1}\})$ tends to $0$ is the same. It suffices to apply lemma \ref{lemma:optim_cond} with $t = T_{k-1}$ and $t = T_{k}$ to split the event in $3$ sub-events and follow the proof. By symmetry, we also have $\mathbb{P}(E_{n,k,2})\rightarrow 0$. 

Let's now focus on $E_{n,k,3}$. Like in \cite{harchaoui2010multiple}, the event is split is four independent events:
\begin{equation*}
    E_{n,k,3} = E^{(1)}_{n,k,3} \cup E^{(2)}_{n,k,3} \cup E^{(3)}_{n,k,3} \cup E^{(4)}_{n,k,3} 
\end{equation*}
with
\begin{align*}
    & E^{(1)}_{n,k,3} = E_{n,k,3} \cap \{ T_{k-1}<\widehat{T}_l<\widehat{T}_{l+1}<T_{k+1} \} \\
    & E^{(2)}_{n,k,3} = E_{n,k,3} \cap \{ T_{k-1}<\widehat{T}_l<T_{k+1}, \widehat{T}_{l+1} > T_{k+1} \} \\
    & E^{(3)}_{n,k,3} = E_{n,k,3} \cap \{\widehat{T}_l<T_{k-1} ,T_{k-1}<\widehat{T}_{l+1}<T_{k+1} \} \\
    & E^{(4)}_{n,k,3} = E_{n,k,3} \cap \{\widehat{T}_l<T_{k-1} ,\widehat{T}_{l+1}>T_{k+1} \}.
\end{align*}
To prove that each one of the previous events have a probability that tends to $0$ as $n$ grows, we invite the reader to read the proof of \cite{harchaoui2010multiple}. It consist in multiple applications of the different Lemmas, the same way we used them in the previous part. Only the time at which lemma \ref{lemma:optim_cond} is used changes and are given by \cite{harchaoui2010multiple}. This concludes the proof.
\end{proof}

\subsection*{Supplementary Lemmas}
Below, the different lemmas necessary to prove the main results are given.
\begin{lemma}
\label{lemma:lambda_borne}
Let $\{x^{(i)}\}_{i=1}^n$ be a set of i.i.d observation sampled from an Ising model with parameter $\Theta \in \mathbb{R}^{p\times p}$ and assume that assumption (A1) is satisfied. Then, $\forall r,l \in [n]$ such that $l<r$ and $r-l>v_n$ with $v_n$ a positive serie, we have $\forall \epsilon>0$:
\begin{align}
    \mathbb{P}\left(\Lambda_{\min}\left(\frac{1}{r-l+1}\sum_{i=l}^rx_{\smallsetminus a}^{(i)}x_{\smallsetminus a}^{(i)\top}\right)\leq \phi_{\min}-\epsilon\right) \nonumber \\ 
    \leq 2(p-1)^2\exp{\left(-\frac{\epsilon^2v_n}{2}\right)}
    \label{eq:phi_min_det}
\end{align}
and
\begin{align}
    \label{eq:phi_max_det}
    \mathbb{P}\left(\Lambda_{\max}\left(\frac{1}{r-l+1}\sum_{i=l}^rx_{\smallsetminus a}^{(i)}x_{\smallsetminus a}^{(i)\top}\right) \geq \phi_{\max}+\epsilon\right) \nonumber \\
    \leq 2(p-1)^2\exp{\left(-\frac{\epsilon^2v_n}{2}\right)}.
\end{align}
\end{lemma}
\begin{proof}
Let $\widehat{\Sigma}=\frac{1}{r-l+1}\sum_{i=l}^r x_{\smallsetminus a}^{(i)}x_{\smallsetminus a}^{(i)\top}$ and $\Sigma=\mathbb{E}\left[X_{\smallsetminus a}X_{\smallsetminus a}^{\top}\right]$. \\ 

We first proove the inequality (\ref{eq:phi_min_det}). Recall that for a symmetric matrix $M$, we have $\Lambda_{\max}(M)\leq \norm{M}_F$, the Frobenius norm of $M$. We have
\begin{align}
    \Lambda_{\min}(\widehat{\Sigma}) & = \min_{\norm{v}_2=1}v^{\top} \widehat{\Sigma} v \\
    & \geq \min_{\norm{v}_2=1}v^{\top} \Sigma v - \max_{\norm{v}_2=1}v^{\top} (\widehat{\Sigma} - \Sigma) v \\
    & \geq \Lambda_{\min}(\Sigma) - \Lambda_{\max}(\widehat{\Sigma} - \Sigma) \\
    & \geq \phi_{\min} - \norm{\widehat{\Sigma} - \Sigma}_F.
\end{align}

Let $s_{mq}^{(i)}$ be the $(m,q)$-th coordinate of $x_{\smallsetminus a}^{(i)}x_{\smallsetminus a}^{(i)\top} - \Sigma$ and $\frac{1}{r-l+1}\sum_{i=l}^r s_{mq}^{(i)}$ the one of $\widehat{\Sigma} - \Sigma$. Note that $\mathbb{E}\left[s_{mq}^{(i)}\right] = 0$ and $|s_{mq}^{(i)}|\leq 2$. Let us analyze the quantity $\mathbb{P}\left(\norm{\widehat{\Sigma} - \Sigma}_F>\epsilon\right)$ with $\epsilon>0$:

\begin{align}
    \mathbb{P}\left(\norm{\widehat{\Sigma} - \Sigma}_F>\epsilon\right) &= \mathbb{P}\left((\sum_{m,q}s_{mq}^2)^{1/2}>\epsilon\right) \\
    & = \mathbb{P}\left(\sum_{m,q}s_{mq}^2>\epsilon^2\right) \\
    & \leq \sum_{m,q}\mathbb{P}\left(s_{mq}^2>\epsilon^2\right) \\
    & \leq \sum_{m,q}\mathbb{P}\left(|s_{mq}|>\epsilon\right). \label{eq:union_bound}
\end{align}

Thanks to Hoeffding's inequality, we have $\mathbb{P}\left(|s_{mq}|>\epsilon\right) \leq 2\exp{\left(-\frac{\epsilon^2(r-l+1)}{2}\right)}$. Since $r-l>v_n$, we also have $\mathbb{P}\left(|s_{mq}|>\epsilon\right) \leq 2\exp{\left(-\frac{\epsilon^2v_n}{2}\right)}$. It follows from (\ref{eq:union_bound}) that $\mathbb{P}\left(\norm{\widehat{\Sigma} - \Sigma}_F>\epsilon\right) \leq 2(p-1)^2\exp{\left(-\frac{\epsilon^2v_n}{2}\right)}$. We deduce that:
\begin{equation}
     \mathbb{P}\left(\!\Lambda_{\min}(\widehat{\Sigma}) \geq  \phi_{\min}-\epsilon\right) \geq 1- 2(p-1)^2\exp{\left(\!-\frac{\epsilon^2v_n}{2}\right)},
\end{equation}
which concludes the proof for (\ref{eq:phi_min_det}). \\

To prove (\ref{eq:phi_max_det}) it suffices to note that $\Lambda_{\max}(\widehat{\Sigma}) \leq \phi_{\max}+ \norm{\widehat{\Sigma} - \Sigma}_F$ and use the same arguments.
\end{proof}

\begin{lemma}
\label{lemma:lambda_borne_rand}
Let $\{x^{(i)}\}_{i=1}^n$ be a set of i.i.d observation sampled from an Ising model with parameter $\Theta \in \mathbb{R}^{p\times p}$ and assume that assumption (A1) is satisfied. \\ 
\noindent
Let $R$ and $L$ be two random variable such that $ R,L\in [n]$,  $L<R$ and $R-L>v_n$ almost surely, with $v_n$ a positive serie. For a fixed node $a$ and any $\epsilon>0$, there exist a constant $c_1>0$ such that:
\begin{align}
     \!\!\!\mathbb{P}\left(\Lambda_{\min}\left(\frac{1}{R-L+1}\sum_{i=L}^Rx_{\smallsetminus a}^{(i)}x_{\smallsetminus a}^{(i)\top}\right) \leq \phi_{\min}-\epsilon\right) \nonumber
    \\
     \!\!\!\leq c_1\exp{\left(-\frac{\epsilon^2v_n}{2} + 2\log(n)\right)} \label{eq:phi_min_rand}
\end{align}
and
\begin{align}
    \mathbb{P}\left(\Lambda_{\max}\left(\frac{1}{R-L+1}\sum_{i=L}^Rx_{\smallsetminus a}^{(i)}x_{\smallsetminus a}^{(i)\top}\right) \geq \phi_{\max}+\epsilon\right) \nonumber \\
    \leq c_1\exp{\left(-\frac{\epsilon^2v_n}{2} + 2\log(n)\right)}.   \label{eq:phi_max_rand}
\end{align}
\end{lemma}
\begin{proof}
We note $\widehat{\Sigma}(L,R)=\frac{1}{R-L+1}\sum_{i=L}^R x_{\smallsetminus a}^{(i)}x_{\smallsetminus a}^{(i)\top}$ and $\mathcal{I}\triangleq\left\{(l,r)\in[n]^2 : r-l>v_n \right\}$. \\
We first prove the inequality (\ref{eq:phi_min_rand}):
\begin{align}
    & \mathbb{P}\left(\Lambda_{\max}\left(\widehat{\Sigma}(L,R)\right) \geq \phi_{\max}+\epsilon\right) \\
    & = \sum_{(l,r)\in \mathcal{I}}\mathbb{P}\left(\Lambda_{\max}\left(\widehat{\Sigma}(L,R)\right), L=l,R=r\right) \\
    & \leq \sum_{(l,r)\in \mathcal{I}}\mathbb{P}\left(\Lambda_{\max}\left(\widehat{\Sigma}(L,R)\right) \middle| L=l,R=r\right). \label{cond_1}
\end{align}
Using Lemma\,\ref{lemma:lambda_borne} we can bound (\ref{cond_1}):
\begin{align}
    (\ref{cond_1}) & \leq \sum_{(l,r)\in \mathcal{I}}   2(p-1)^2\exp{\left(-\frac{\epsilon^2v_n}{2}\right)}\\
    & \leq |\mathcal{I}| c_1\exp{\left(-\frac{\epsilon^2v_n}{2}\right)}\\
    & \leq n^2 c_1\exp{\left(-\frac{\epsilon^2v_n}{2}\right)} \\
    & \leq c_1\exp{\left(-\frac{\epsilon^2v_n}{2} + 2\log(n)\right)} 
\end{align}
with $c_1=2(p-1)$. This concludes the proof for (\ref{eq:phi_min_rand}). Same arguments are used to prove $(\ref{eq:phi_max_rand})$.
\end{proof}

\begin{lemma}
\label{lemma:concentration_R3}
Let $\{x^{(i)}\}_{i=1}^n$ be a set of independent observation sampled from the time-varying Ising model (Section 2). Then, $\forall j \in [D]$ and $\forall r,l \in \{T_j, \ldots, T_{j+1}-1\}$ such that $l<r$, we have:
\begin{align}
\label{eq:concentration_R3}
    \mathbb{P}\left(\frac{1}{r-l+1}\norm{R_3(l,r)}_2 \leq  2\sqrt{\frac{p\log(n)}{r-l+1}} \right) \\
    \geq 1 - 2(p-1)\exp{\left(-2p\log(n)\right)}
\end{align}
with $R_3(l,r) = \sum_{i=l}^{r}  x^{(i)}_{\smallsetminus a}\left\{x^{(i)}_a - \mathbb{E}_{\Theta^j}\left[X_a | X_{\smallsetminus a} = x^{(i)}_{\smallsetminus a} \right]\right\}.$
\end{lemma}
\begin{proof}
Let $Z_{ij}$ be the the $j$-th element of the vector \\
$\frac{1}{r-l+1} x^{(i)}_{\smallsetminus a}\left\{x^{(i)}_a - \mathbb{E}_{\Theta}\left[X_a | X_{\smallsetminus a} = x^{(i)}_{\smallsetminus a} \right]\right\}$. Note that $|Z_{ij}|\leq \frac{2}{r-l+1}$ and $\mathbb{E}\left[Z_{ij}\right] = 0$. Let $\epsilon>0$, we have:
\begin{align*}
    \mathbb{P} & \left(\frac{1}{r-l+1} \norm{R_3(l,r)}_2 \geq  \epsilon\right) \\
    &= \mathbb{P}\left(\sqrt{\sum_{j\neq a}(\sum_{i=l}^r Z_{ij})^2} \geq  \epsilon\right) \\
    & = \mathbb{P}\left(\sum_{j\neq a}(\sum_{i=l}^r Z_{ij})^2 \geq  \epsilon^2\right) \\ 
    & \leq \sum_{j\neq a}  \mathbb{P}\left(|\sum_{i=l}^r Z_{ij}| \geq  \epsilon\right) \\
    & \leq 2(p-1)\exp{\left(-\frac{\epsilon^2(r-l+1)}{2}\right)}.
\end{align*}
Now, if we fix $\epsilon=2\sqrt{\frac{p\log(n)}{r-l+1}}$, we obtain:
\begin{align*}
    \mathbb{P}  \left(\frac{1}{r-l+1}\norm{R_3(l,r)}_2 \leq  2\sqrt{\frac{p\log(n)}{r-l+1}} \right) \\
     \geq 1 - 2(p-1)\exp{\left(-2p\log(n)\right)}.
\end{align*}
\end{proof}

\begin{lemma}
\label{lemma:concentration_R3_random}
Let $\{x^{(i)}\}_{i=1}^n$ be a set of independent observation sampled from the time-varying Ising model (Section 2). We have:
\begin{align}
    \mathbb{P}\left(\underset{j \in [D]}{\bigcap} \underset{l,r \in \mathcal{I}_j}{\bigcap}\left\{\frac{1}{r-l+1}\norm{R_3^j(l,r)}_2 \leq  2\sqrt{\frac{p\log(n)}{r-l+1}}\right\} \right) \nonumber \\ 
    \geq 1 - c_2\exp{\left(-c_3\log(n)\right)} 
    \label{eq:concentration_R3_tot}
\end{align}
with $R_3^j(l,r) = \sum_{i=l}^{r}  x^{(i)}_{\smallsetminus a}\left\{x^{(i)}_a - \mathbb{E}_{\Theta^j}\left[X_a | X_{\smallsetminus a} = x^{(i)}_{\smallsetminus a} \right]\right\} $, $c_2,c_3$ some positive constants and $\mathcal{I}_j\triangleq\left\{(l,r)\in \{T_j,\ldots, T_{j+1}-1\}^2 : r>l \right\}$.
\end{lemma}
\begin{proof}
The proof is a simple application of Lemma \ref{lemma:concentration_R3}:
\begin{align*}
    &\mathbb{P}\left(\underset{j \in [D]}{\bigcup} \underset{l,r \in \mathcal{I}_j}{\bigcup}  \left\{\frac{1}{r-l+1}\norm{R_3^j(l,r)}_2 \geq  2\sqrt{\frac{p\log(n)}{r-l+1}}\right\}\right) \\
    &  \leq \sum_{j \in [D]} \sum_{l,r \in \mathcal{I}_j} \mathbb{P}\left(\frac{1}{r-l+1}\norm{R_3^j(l,r)}_2 \geq  2\sqrt{\frac{p\log(n)}{r-l+1}}\right) \\
    & \leq 2Dn^2(p-1)\exp{\left(-2p\log(n)\right)} \\
    & \leq c_2\exp{\left(-2p\log(n)+2\log(n)\right)}\\
    & \leq c_2\exp{\left(-c_3\log(n)\right)}
\end{align*}
since $p>1$. This concludes the proof.
\end{proof}

\bibliographystyle{mystyle}
{\small
\bibliography{biblio}
}